\documentclass[twoside,11pt]{article}
\usepackage{ISG_style_reduced}
\usepackage[utf8]{inputenc} 
\usepackage[T1]{fontenc}    
\usepackage[hidelinks]{hyperref}       
\usepackage{url}            
\usepackage{booktabs}       
\usepackage{amsfonts,amssymb}       
\usepackage{nicefrac}       
\usepackage{microtype}      
\usepackage{amssymb}
\usepackage{physics}        
\usepackage{times}
\usepackage{graphicx}
\usepackage{subcaption}
\usepackage{bbm}
\usepackage{dsfont}
\usepackage{flushend}
\usepackage{float}
\usepackage[dvipsnames]{xcolor}
\usepackage{algorithm}
\usepackage{algpseudocode}
\usepackage{mathtools}
\usepackage{aligned-overset}
\usepackage{natbib}
\usepackage{comment}
\usepackage{tablefootnote}
\usepackage{footmisc}

\usepackage{tikz}
\usepackage{titling}
\preauthor{\begin{flushleft}}
\postauthor{\end{flushleft}}

\usetikzlibrary{positioning}

\DeclareMathOperator{\image}{im}
\DeclareMathOperator{\nondesc}{nondesc}
\DeclareMathOperator{\sur}{sur}

\DeclareMathOperator{\sgn}{sign}

\def\ci{\perp\!\!\!\perp}

\def \sdot {\dot \sigma}

\title{\bf \LARGE Score-based Causal Representation Learning with Interventions
}

\author{{\bf Burak Var\i c\i} \hfill \texttt{varicb@rpi.edu}\\
\textit{Rensselaer Polytechnic Institute} \vspace{.1 in}\\
{\bf Emre  Acart\"urk} \hfill \texttt{acarte@rpi.edu}\\
\textit{Rensselaer Polytechnic Institute} \vspace{.1 in}\\
{\bf Karthikeyan  Shanmugam }\hfill{\texttt{karthikeyanvs@google.com}}\\
\textit{Google} \vspace{.1 in}\\
{\bf Abhishek Kumar }\hfill{\texttt{abhishk@google.com}}\\
\textit{Google} \vspace{.1 in}\\
{\bf Ali Tajer} \hfill{\texttt{tajer@ecse.rpi.edu}}\\
\textit{Rensselaer Polytechnic Institute}
}
\date{}

\begin{document}
\maketitle

\begin{abstract}
This paper studies the causal representation learning problem when the latent causal variables are observed indirectly through an unknown linear transformation. The objectives are: (i) recovering the unknown linear transformation (up to scaling) and (ii) determining the directed acyclic graph (DAG) underlying the latent variables. Sufficient conditions for DAG recovery are established, and it is shown that a large class of non-linear models in the latent space (e.g., causal mechanisms parameterized by two-layer neural networks) satisfy these conditions. These sufficient conditions ensure that the effect of an intervention 
can be detected correctly from changes in the score. Capitalizing on this property, recovering a valid transformation is facilitated by the following key property: any valid transformation renders latent variables' score function to necessarily have the minimal variations across different interventional environments. This property is leveraged for perfect recovery of the latent DAG structure using only \emph{soft} interventions. For the special case of stochastic \emph{hard} interventions, with an additional hypothesis testing step, one can also uniquely recover the linear transformation up to scaling and a valid causal ordering. 
\end{abstract}

\section{Introduction}\label{sec:introduction}

\emph{Disentangled representation learning} aims to infer a latent representation such that each dimension of the latent representation corresponds to a meaningful or controllable component of the observed data or the process generating the data. Despite the discrepancies in objective and approach, the common assumption in the existing literature is that the disentangled latent variables are statistically independent. Such independence is enforced to facilitate inferring representations~\citep{higginslearning,kumar2018,chen2018isolating,kim2018disentangling}. However, \emph{identifiability} (of the inverse of data generating transformation) is known to be impossible without additional supervision or non-identically distributed data \citep{hyvarinen1999nonlinear,locatello2019challenging}. Approaches for mitigating this issue include using posterior regularization \citep{kumar2020implicit}, requiring the knowledge of auxiliary variables \citep{shu2019weakly,pmlr-v108-khemakhem20a,ahuja2022towards}, temporal information \citep{hyvarinen2017nonlinear}, and the knowledge of the mechanisms that govern the system \citep{ahuja2021properties}. These methods do not consider possible causal relationships over latent factors. In a wide range of domains, however, the latent variables are related causally.

To accommodate causal structures among the latent variables, \emph{causal representation learning} aims to learn a disentangled representation leveraging the modularity property of causal models \citep[p.22]{pearl2009causality}. The latent variables are generated by multiple causal mechanisms, i.e., conditional probability kernels. Each mechanism governs the structure among a different group of latent variables that act as an independent module composed by a causal ordering \citep{scholkopf2022causality}. This forces the latent variables to vary in a correlated way in any given environment. However, upon an intervention, \emph{only} the manipulated causal mechanisms relating to that intervention change. These changes, which are generally sparse, suggest that interventions are instrumental to learning disentangled latent representations that are commensurate with the modularity property. 
This has been pointed out as an important open problem in the literature \citep{scholkopf2021toward,scholkopf2022statistical}. Practical applications of latent causal representations have been explored recently in domains such as robotics \citep{lee2021causal,weichwald2022learning} and gene expression modeling \citep{lotfollahi2021compositional}. For instance, in \citep{weichwald2022learning}, for a robot arm control competition, winning entries were observed first to learn a model that is consistent with the observed effects of various interventions and then perform optimal control that uses these models.

In causal representation learning, the latent variables are not observed directly. Rather, they undergo an unknown transformation, and such an unknown transformation generates the observed data. Hence, any causal representation learning approach involves two objectives: (i) \textbf{identifiable representation learning}, which is responsible for identifying the inverse of the unknown transformation to recover the latent variables from observations; and (ii)~\textbf{causal structure learning}, which determines the directed acyclic graph (DAG) underlying the cause-effect relationships among the recovered latent variables. As in disentangled learning, the \emph{identifiability} of causal representation learning requires additional information, and it becomes viable when performed in conjunction with \emph{interventions} \citep{scholkopf2021toward}. Specifically, the interventions can create proper variations in the observed data, which in turn, facilitate identifiability. As a form of strong supervision, several studies used contrastive pairs of -- before and after an intervention -- the same observations \citep{ahuja2022weakly,locatello2020weakly,kulelgen2021,Yang_2021_CVPR}. A representative example is the images of an object from different views \citep{brehmer2022weakly}. From a causal point of view, we can consider these observations as counterfactual pairs. Another approach is using temporal sequences to identify causal variables under interventions \citep{lachapelle2022disentanglement,yao2022,lippe2022icitris}.

In this paper, we focus on the identifiability of latent causal representations. Our approach requires a weaker form of supervision according to which the data under different interventional distributions in the latent space is observable, while the counterfactual instances are not. Our central contribution is establishing that latent causal representations are identifiable under linear transformations and single-node stochastic hard interventions that cover all variables in the latent space.

We start by observing that score functions (i.e., the gradients of log-likelihood functions) generally have sparse changes in their coordinates upon a single-node intervention. We require that the effect of an intervention is not lost for any linear combination of varying coordinates. We show that this condition holds for sufficiently non-linear additive causal models such as two-layer neural networks (NNs) with full-rank first layer. Under this assumption, the pivotal technical idea is that the score functions of transformed interventional distributions have minimal variations across distributions only for the (approximately) correct inverting transformation. Our main contributions are summarized as follows:
\begin{itemize}
    \item For linear transformations and under an exhaustive set of single-node stochastic hard interventions, we identify the transformation up to coordinate-wise scaling and a permutation consistent with a valid causal order. This recovers the latent DAG as well. 

    \item For single-node soft interventions on every node, we establish the identifiability of a transformation up to a non-trivial equivalence class. Surprisingly, this equivalence class is sufficient to recover the latent DAG up to a permutation consistent with the topological order. The recovered latent variables are Markov with respect to a DAG that is isomorphic to the true latent DAG. It is noteworthy that the existing studies on causal representation learning require hard interventions to achieve this result.
\end{itemize}
Our settings and results have key differences from those of the recent studies on causal representation learning \citep{liu2022identifying,seigal2022linear,ahuja2022interventional}. 
\begin{itemize}
    \item {\bf Setting:} We provide identifiability results for non-linear causal relationships in the latent space. This is in contrast to the setting of \citep{liu2022identifying}, which investigates the intrinsic indeterminacies in the latent space where the causal relationships are linear. Similarly, \citep{seigal2022linear} focuses only on linear Gaussian latent models with linear mapping to observations. 
    \item {\bf Latent DAG recovery:} We perfectly recover latent DAG via soft interventions by leveraging the non-linearity in latent space. Secondly, we can accommodate \emph{stochastic} hard and soft interventions with unbounded support. In contrast, \citep{seigal2022linear} shows that soft interventions are insufficient for recovering latent DAGs with linear causal relationships. In a different direction, \citep{ahuja2022interventional} considers a more general polynomial mapping from the latent to the observational space and provides results for stochastic interventions when the latent random variables are bounded.
\end{itemize}

\begin{table}[t]
    \centering
    \begin{tabular}{rlllll}
        \toprule
        & Latent DAG & Transform
            & Interventions & Latent DAG & Transformation \\
            & & & & recovery & identifiability \\
        \hline 
        \hline 
        \citet{seigal2022linear} & Linear SEM & Linear
            & Hard & Yes & Scaling consistency \\
            \midrule 
        \citet{ahuja2022interventional} & Not restricted & Polynomial
            & None & No & Affine transform \\
        &  Not restricted  & Polynomial & \emph{do} & Yes & {Scaling consistency} \\
        & Bounded variables & Polynomial & Soft & Yes & {Scaling consistency} \\
        \hline
        This paper & Non-linear\tablefootnote{\label{footnote:non-linear}Non-linearity conditions for our results which hold for typical models are discussed in Section~\ref{sec:analyze-assumption}.} & Linear
            & Soft & Yes & Mixing consistency \\
        & Non-linear\footref{footnote:non-linear} & Linear & Hard & Yes & Scaling consistency \\
        \bottomrule
    \end{tabular}
  \caption{Comparison of the results to prior studies in different settings. Scaling consistency refers to recovering representations up to a consistent scaling under a valid causal order. Mixing consistency is a weaker form of scaling. Formal definitions are given in Section~\ref{sec:statement}.}
    \label{tab:related-comp}
\end{table} 

\section{Related Work}\label{sec:related}

\paragraph{Identifiable representation learning.}
A key objective of representation learning is to identify the latent factors that generate observational data. However, several studies have established that provable identification of latent factors is impossible without auxiliary information or additional structure on the data generation process \citep{hyvarinen1999nonlinear,locatello2019challenging}. Several approaches exist to mitigate this issue when there is no causal relationship among the latent factors. Some representative approaches include using posterior regularization for addressing the non-uniqueness of latent variables \citep{kumar2020implicit}, leveraging the knowledge of the mechanisms that govern the evolution of the system \citep{ahuja2021properties}, and using weak supervision with auxiliary information \citep{shu2019weakly}. Finally, non-linear independent component analysis (ICA) uses side information, in the form of structured time-series to exploit temporal information \citep{hyvarinen2017nonlinear,halva2020hidden} or knowledge of auxiliary variables that renders latent variables conditionally independent \citep{pmlr-v108-khemakhem20a,khemakhem2020ice,hyvarinen2019nonlinear}. On a related problem, identifiability of deep generative models is studied without auxiliary information \citep{kivva2022identifiability}.

\paragraph{Causal representation learning.} As a form of strong supervision, several studies have investigated causal representation learning when pairs of observations are available -- one before and one after a mechanism change (e.g., an intervention) for the same underlying realization of exogenous variables involved \citep{ahuja2022weakly,locatello2020weakly,kulelgen2021,Yang_2021_CVPR,brehmer2022weakly}. A representative example is the images of an object from different views \citep{brehmer2022weakly}. From a causal point of view, we can consider these pairs as counterfactual pairs. Another approach is using temporal sequences to identify causal variables under interventions \citep{lachapelle2022disentanglement,yao2022,lippe2022icitris}. However, we do not consider time-series data in this paper. Our approach requires a weaker form of supervision where data under different interventional distributions in the latent space is observable, whereas the counterfactual instances are not observable. 

\paragraph{Interventional causal representation learning.}
The study in \citep{perry2022causal} uses sparse changes on causal mechanisms via soft interventions and recovers the true DAG over observed variables by minimizing the number of conditionals that change across pairs of environments. However, it does not consider latent representations. The study in \citep{liu2022identifying} aims to learn latent causal graphs and identify latent representations. However, its focus is on linear Gaussian latent models, and its extensions to even non-linear Gaussian models are viable at the expense of restricting the graph structure. The study in \citep{seigal2022linear} considers a linear latent model with a linear mapping to observations. Specifically, by analyzing the precision matrices of observations, it recovers latent factors up to scaling and permutations consistent with the true causal order under stochastic hard interventions (perfect interventions in their notations). Finally, the study in \citep{ahuja2022interventional} considers a significantly more general setting in which the observations are polynomial functions of latent variables with no restrictions on the latent causal model. By minimizing the reconstruction loss, under $do$ interventions or stochastic interventions with boundedness and independent support assumptions, \citep{ahuja2022interventional} recovers the latent representation up to scaling and permutation.

\paragraph{Score functions in causality.}  The study in \citep{rolland2022score} uses score-matching to recover non-linear additive Gaussian noise models. The proposed method finds the topological order of causal variables but requires additional pruning to recover the full graph. \citep{montagna2022scalable} focuses on the same setting, recovers the full graph from Jacobian scores, and dispenses with the computationally expensive pruning stage. Both of these studies are limited to observed causal variables, whereas in our case, we have a causal model in the latent space.

\paragraph{Scores and diffusion modeling.} Score-based methods have recently shown impressive results in image generative modeling \citep{ho2020denoising,vahdat2021score} and inverse problems in imaging \citep{chung2022score}. The study in \citep{ho2020denoising} has uncovered a strong connection between probabilistic diffusion models and score-based generative modeling. It has shown that evidence lower bound (ELBO) for training diffusion models is essentially equivalent to a weighted combination of score-matching objectives and demonstrated strong empirical performance, e.g., generating high-quality image samples. The study in \citet{song2021score} has further investigated this connection and proposed a unified framework for score-based diffusion models.

\section{Problem Setting}\label{sec:problem}

\paragraph{Notations.} For a vector $a$, the $i$-the entry is denoted by $a_i$ and $[a]_i$. For a matrix $A$, the $i$-th row is denoted by $[A]_{i}$, the entry at row $i$ and column $j$ is denoted by $[A]_{i,j}$, and $\image(A)$ denotes the image of $A$. For matrix $A\in\R^{m\times n}$, $\mathds{1}(A)\in\{0,1\}^{m\times n}$ is the indicator function whose values are 1 at non-zero entries of $A$. For matrices $A$ and $B$ that have the same shapes, $A \preccurlyeq B$ denotes component-wise inequality. $I_n$ denotes the $n \times n$ identity matrix. For a positive integer $n$, we define $[n] \triangleq \{1,\dots,n\}$. 

\paragraph{The data generating process.} 
Consider a process that generates random variables $X \triangleq [X_1,\dots,X_d]^{\top}$ based on underlying latent random variables $Z \triangleq  [Z_1,\dots,Z_n]^{\top}$. We denote the probability density functions (pdfs) of $X$ and $Z$ by $p_X$ and $p_Z$, respectively. For clarity in the analysis, these pdfs are assumed to be well-defined, and the distributions of $X$ and $Z$ are absolutely continuous with respect to the Lebesgue measure. We focus on a linear data-generation process, in which the latent random vector $Z$ is linearly mapped to observations $X$ through a \emph{transformation} matrix $T \in \R^{d \times n}$ as follows.
\begin{align}
    X &= T \cdot Z \ , \qquad  \mbox{where} \qquad T \in \R^{d \times n} \ . \label{eq:data-generation-process}
\end{align}
We assume that $d \geq n$ and $\rank(T)=n$. Otherwise, identifiable recovery, which we will formalize in Section~\ref{sec:statement}, is ill-posed. 

\paragraph{Latent causal structure.} The distribution of latent variables $Z$ factorizes with respect to a DAG represented by $\mcG_{Z}$ that consists of $n$ nodes. Node $i\in[n]$ of $\mcG_{Z}$ represents $Z_i$ and $p_Z$ factorizes according to: 
\begin{align}
    p_{Z}(z) = \prod_{i=1}^{n} p_{Z}(z_i \med z_{\Pa(i)}) \ , \label{eq:pz_factorized}
\end{align}
where $\Pa(i)$ denotes the set of parents of node $i$. For each node $i\in[n]$, we also define $\overline{\Pa}(i) \triangleq \Pa(i) \cup \{i\}$. Similarly, $\Ch(i)$ denotes the set of children of node $i$, and $\overline{\Ch}(i)\triangleq \Ch(i) \cup \{i\}$.
Based on the modularity property, a change in the causal mechanism of node~$i$ does not affect those of the other nodes. We also assume that all conditional pdfs $\{p_Z(z_i \mid z_{\Pa(i)}) : i \in[n])\}$ are continuously differentiable with respect to all $z$ variables and $p_Z(z) \neq 0$ for all $z \in \R^n$. 

\paragraph{Interventions and environments.} We assume that for each node $i\in[n]$, in addition to observational mechanism $p_{Z}(z_i \med z_{\Pa(i)})$, there also exists an interventional mechanism $q_{Z}(z_i \med z_{\Pa(i)})$, which is assumed to be different from $p_{Z}(z_i \med z_{\Pa(i)})$. We consider $M$ \emph{interventional} environments such that in each environment, a subset of nodes are intervened. We denote these environments by $\mcE\triangleq \{\mcE^{m}\;:\; m\in[M]\}$, and denote the set of nodes intervened in environment $\mcE^{m}$ by $I^{m}\subseteq [n]$. We also adopt the convention that $I^{0} = \emptyset$, and define $\mcE^{0}$ as the \emph{observational} environment. 
The set of all intervention sets is denoted by $\mcI \triangleq \{I^{m}\;:\; m\in[M]\cup\{0\}\}$. We consider two types of interventions.
\begin{itemize}
    \item {\bf Soft Interventions:} A soft intervention on node $i$ does not necessarily remove the functional dependence of an intervened node on its parents, and rather alters it to a different mechanism. Under soft intervention for each $m\in[M]$, the latent random vector $Z$ in $\mcE^{m}$ is generated according to:
    \begin{align}
        p^{m}_{\rm s}(z) = \prod_{i \notin I^{m}} p_Z(z_i \med z_{\Pa(i)}) \prod_{i \in I^{m}} q_Z(z_i \med z_{\Pa(i)}) \ . \label{eq:pz_m_factorized}
    \end{align}
    \item {\bf Hard Interventions:} A hard intervention on node $i$ removes the edges incident on $i$, and changes its causal mechanism 
    to $q_Z(z_{i})$. Under hard intervention for each $m\in[M]$, the latent random vector $Z$ in $\mcE^{m}$ is generated according to:
\begin{align}
    p^{m}_{\rm h}(z) & =  \prod_{i \not\in I^{m}} p_Z(z_{i} \mid z_{\Pa(i)}) \prod_{i \in I^{m}} q_Z(z_i)\ . \label{eq:pz_m_factorized_hard}
\end{align}
\end{itemize}

\paragraph{Score function.} Define the \emph{score function} associated with a probability distribution as the gradient of its pdf. We denote the score functions associated with the distributions of the observed variables and the latent variables by $s_{X}$ and $s_{Z}$, respectively. Hence,
\begin{align}
    s_{X}(x) & \triangleq \nabla_{x} \log p_{X}(x) \ , \qquad 
    \mbox{and} \qquad s_{Z}(z)  \triangleq \nabla_z \log p_Z(z) \ . \label{eq:score-definition}
\end{align}
Leveraging the causal structure $\mcG_{Z}$ and the associated factorization in~\eqref{eq:pz_factorized}, the score function $s_{Z}(z)$ decomposes as
\begin{align}
 s_{Z}(z) & = \sum_{i=1}^{n} \nabla_z \log  p_Z(z_i \med z_{\Pa(i)}) \ . \label{eq:sz_decompose}
\end{align}
Similarly to \eqref{eq:sz_decompose}, by invoking \eqref{eq:pz_m_factorized}, the score function under a soft intervention in environment $\mcE^{m}$ has the following decomposition:
 \begin{align}
    \label{eq:sz_m_decompose}    
    s^{m}_Z(z) &= \nabla_z \log p^{m}_{\rm s}(z) = \sum_{i \notin I^{m}} \nabla_z \log p_Z(z_i \med z_{\Pa(i)}) + \sum_{i \in I^{m}} \nabla_z \log q_Z(z_i \med z_{\Pa(i)})  \ , 
\end{align}
 and by invoking \eqref{eq:pz_m_factorized_hard}, the score function 
 under a hard intervention in environment $\mcE^{m}$ has the following form:
\begin{align}
    \label{eq:sz_m_decompose_hard}    
    s^{m}_Z(z) &= \nabla_z \log p^{m}_{\rm h}(z) = \sum_{i \notin I^{m}} \nabla_z \log p_Z(z_i \med z_{\Pa(i)}) + \sum_{i \in I^{m}} \nabla_z \log q_Z(z_i)  \ .
\end{align}
The score functions change across different environments induced by the changes in the distribution of $Z$. Furthermore, we define $s_X^m$ as the score function of $X$ in the environment $\mcE^m$. In Section~\ref{sec:score-properties}, we will delineate the relationship between $s^m_X$ and $s^m_Z$ for $m\in[M]$.
\begin{definition}[Almost Sure Equivalence]
Given random vector $Y \in \R^r$ with pdf $p_Y$, functions $f, g: \R^r \to\R$ are \emph{almost sure equivalent}, denoted by $f(Y)\overset{p_Y}{=}g(Y)$, if
\begin{align}\label{eq:deq-defn}
 \P(f(Y) = g(Y)) = 1 \ . 
\end{align}
The lack of equivalence is denoted by $f(Y) \overset{p_Y}{\neq} g(Y)$.  
\end{definition}
\begin{figure}[t]
    \centering
    \begin{subfigure}[t]{0.4\textwidth}
        \centering
        \includegraphics[width=0.7\linewidth]{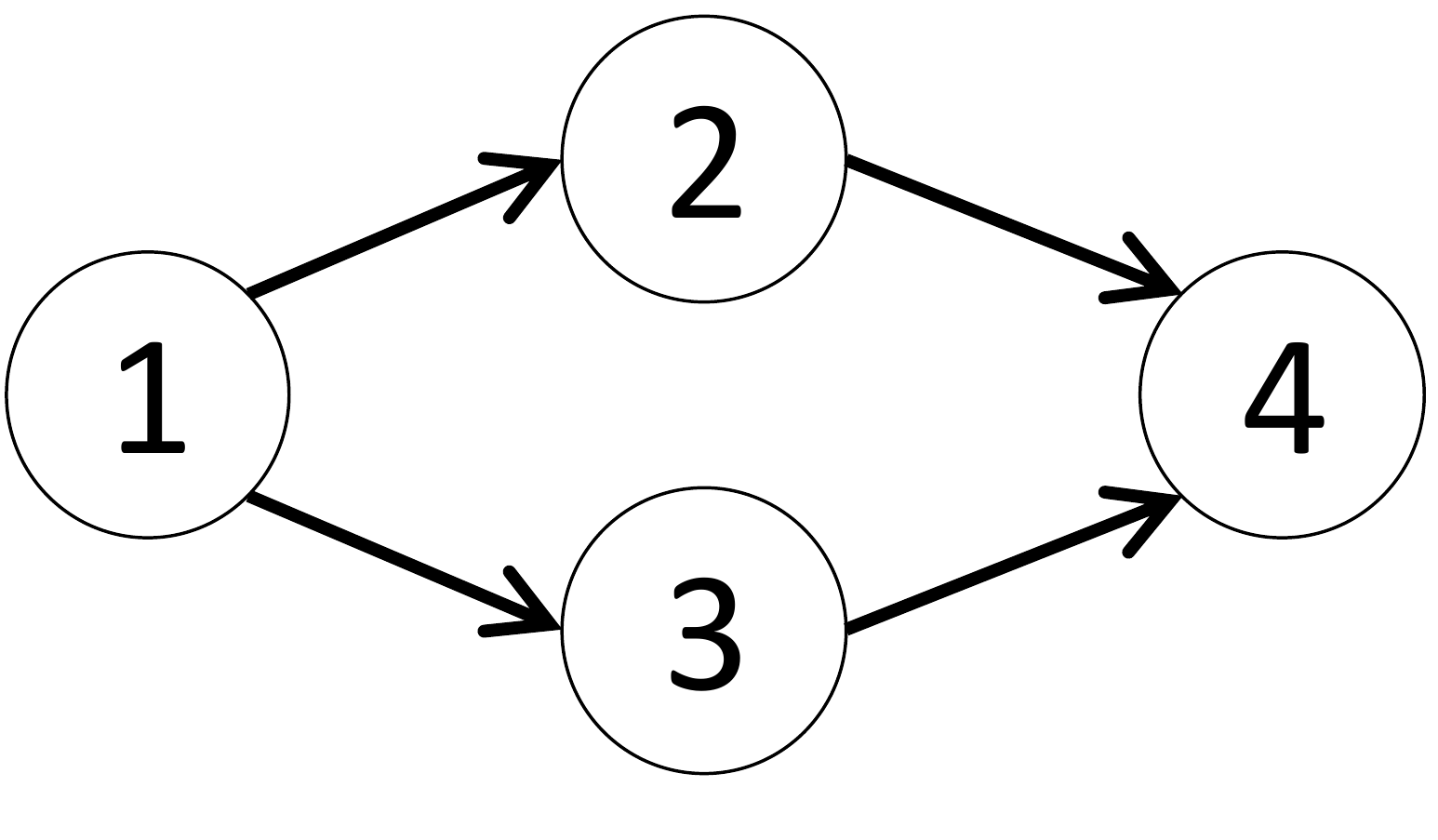}
        \caption{$[1,2,3,4]$ and $[1,3,2,4]$ are valid causal orders, and only surrounded node is $4$ with $\sur(4)=\{2,3\}$.}
        \label{fig:sample-graph-1}
    \end{subfigure}
    \hspace{.2 in}
    \begin{subfigure}[t]{0.4\textwidth}
        \centering
        \includegraphics[width=0.4\linewidth]{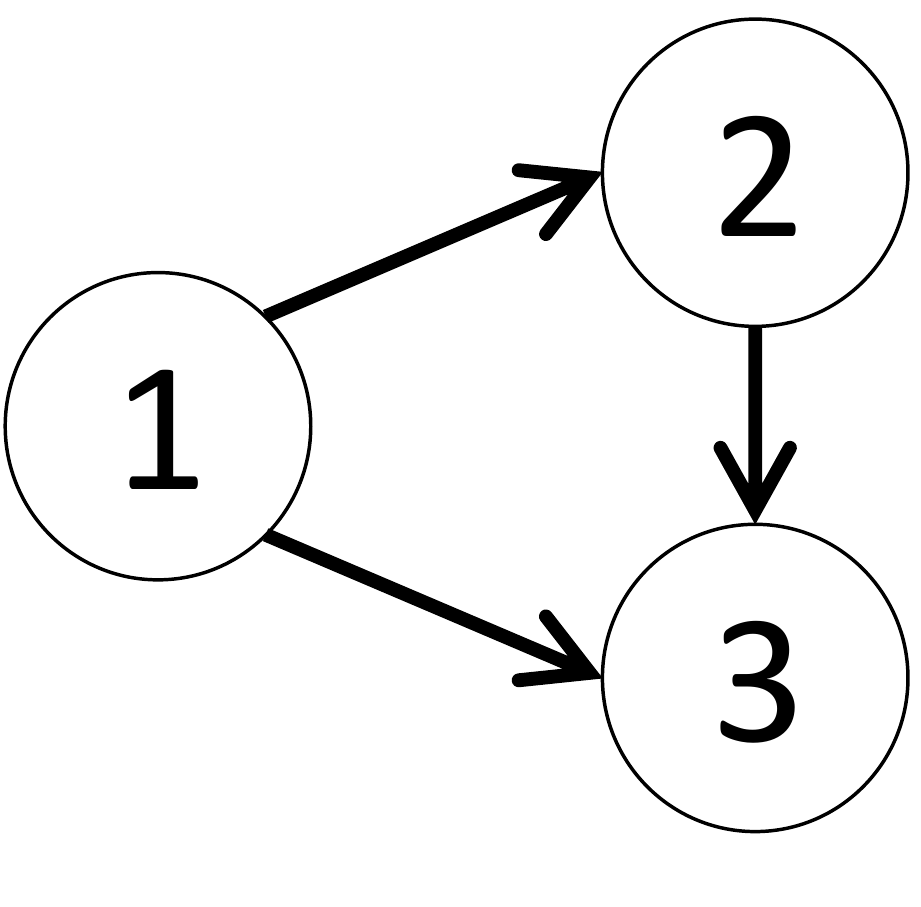}
        \caption{the only valid causal order is $[1,2,3]$, and the surrounded nodes are $\mcS=\{2,3\}$ with $\sur(2)=\{1\}, \sur(3)=\{1,2\}$.}
        \label{fig:sample-graph-2}
    \end{subfigure}   
    \caption{Sample latent DAGs.}
    \label{fig:sample}
\end{figure}

\section{Statement of the Objective}\label{sec:statement}
\paragraph{Recovering latent variables.} 
Our objective is to use observations $X$ and recover the true latent variables $Z$. We denote a generic estimator of $Z$ given $X$ by $\hat Z(X):\R^d\to\R^n$. In order to assess the fidelity of the estimate $\hat Z(X)$ with respect to the ground truth $Z$, we provide the following definitions. We define $\Pi$ as the space of all possible permutation matrices of dimension $n$.
\begin{definition}[Valid Causal Orders]\label{def:valid-causal-order}
We refer to a permutation $\pi\triangleq [\pi_1,\dots,\pi_n]$ of $[1,\dots,n]$ as a \emph{valid causal order} if $\pi_i \in \Pa(\pi_j)$ indicates that $i < j$. Without loss of generality, we assume that $[1,\dots,n]$ is a valid causal order. We denote the permutation matrix associated with permutation $\pi$ by $P_\pi\in\Pi$, i.e., $[\pi_1 \;\;\pi_2\;\; \dots \; \pi_n]^\top=P_\pi\cdot[1 \;\; 2\;\; \dots \;\; n]^\top$.
\end{definition}
\begin{definition}[Surrounded Node] \label{def:surrounded}
Node $i\in[n]$ in DAG $\mcG_{Z}$ is said to be \emph{surrounded} if there exists another node $j\in[n]$ such that $\overline{\Ch}(i) \subseteq \Ch(j)$. We denote the set of nodes that surround $i\in[n]$ by $\sur(i)$, and the set of surrounded nodes by $\mcS$, i.e.,
\begin{align}
    \sur(i) &\triangleq \{j\in[n] \, : \, j\neq i\;\;, \;\; \overline{\Ch}(i) \subseteq \Ch(j) \}\ . \\
        \mcS &\triangleq \{i\in [n] \, : \; \sur(i)\neq \emptyset\}\ . \label{eq:def-surrounded-set}
\end{align}
\end{definition}
Figure~\ref{fig:sample} illustrates definitions of valid causal order and surrounded nodes via two graphical examples.

\paragraph{Recovery fidelity measures.} In recovering the latent variables $Z$, our objective is to ensure the following measures of fidelity between the ground truth $Z$ and the estimate $\hat Z(X)$.

\begin{enumerate}
    \item \textbf{Valid DAG recovery:} We have a \emph{valid DAG recovery} if the distribution of $\hat Z(X)$ factorizes with respect to a DAG $\mcG_{\hat Z}$ that is equal to the true DAG $\mcG_{Z}$ up to a valid causal order $\pi$.
    
    \item \textbf{Scaling consistency:} We say that $\hat Z(X)$ satisfies \emph{scaling consistency} if under a valid causal order $\pi$, the coordinates of $\hat Z(X)$ are consistently equal to those of $Z$ up to fixed scaling factors. Specifically, for any sequence of latent random variables  $\{Z^t\in\R^n : t\in\N\}$ and their associated estimates $\{\hat Z^t\in\R^n : t\in\N\}$ under the causal order $\pi$ with the associated permutation matrix $P_\pi$, we have
    \begin{align}\label{eq:scaling_cons}
        \hat Z^t = P_\pi\cdot C_{\rm s}\cdot  Z^t\ ,
    \end{align}
    where $C_{\rm s}\in\R^{n\times n}$ is a constant \emph{diagonal} matrix accounting for scaling. This property essentially states that the estimates are equal to the ground truth up to a fixed permutation and fixed coordinate-wise scaling that is constant for all realizations and estimates of the latent variables.
    
    \item \textbf{Mixing consistency:} Finally, we relax the scaling consistency by including a mixing component. We say that $\hat Z(X)$ satisfies \emph{mixing consistency} if under a valid causal order $\pi$, the coordinates of $\hat Z(X)$ are consistently equal to linear functions of those of $Z$ and their surrounding variables. Specifically, for any sequence of latent random variables  $\{Z^t\in\R^n : t\in\N\}$ and their associated estimates $\{\hat Z^t\in\R^n : t\in\N\}$  under the causal order $\pi$ with the associated permutation matrix $P_\pi$, we have
    \begin{align}\label{eq:mixing_cons}
        \hat Z^{t}= P_\pi\cdot (C_{\rm s}+C_{\rm m})\cdot Z^{t}  \ ,
    \end{align}
    where $C_{\rm s}\in\R^{n\times n}$ is a constant \emph{diagonal} matrix and $C_{\rm m}\in\R^{n\times n}$ is a constant \emph{sparse} matrix with that accounts for mixing and it satisfies:
    \begin{align}
        j\notin {\rm sur}(i) \quad \Rightarrow \quad [C_{\rm m}]_{i,j}= 0\ , \qquad \forall i\in[n]\ .
    \end{align}
    We note that diagonal elements of $C_{\rm m}$ are zero and when $C_{\rm m}=\boldsymbol{0}_{n\times n}$, mixing consistency in~\eqref{eq:mixing_cons} is strengthened to scaling consistency in~\eqref{eq:scaling_cons}. 
\end{enumerate}
In the algorithm that we will develop, forming the estimate $\hat Z(X)$ is facilitated by first estimating the transformation $T$ and then using that to estimate $Z$ based on the relationship in~\eqref{eq:data-generation-process}. To formalize this process, we first define the set of possible transformations. Note that, when $U$ is a candidate transformation, its Moore-Penrose inverse, i.e., $U^+  \triangleq (U^{\top} U)^{-1}{U}^{\top}$,
can be considered as its associated candidate inverse transformation. Any pair of $U$ and $U^+$ that does not reconstruct $X$ correctly is trivially incorrect. Therefore, by enforcing the perfect reconstruction of $X$, the set of possible transformations with rank $n$ is defined as
\begin{align}\label{eq:reconstruction-prop}
    \mcU \triangleq \{ &U \in \R^{d \times n} : \rank(U)=n  \quad \mbox{and} \quad U U^{+} X = X \,, \quad \forall X \in \image(T) \} \ .
\end{align}
For any given $X$ and candidate transformation $U\in\mcU$, 
in addition to $\hat Z(X)$ that is our ultimate estimate for $Z$, we also define an \emph{auxiliary} estimate $\hat Z(U,X):\mcU\to\R^n$ as a function of $U$ and $X$:
\begin{align}\label{eq:encoder}
    \hat Z(U,X) = U^+ \cdot X\ , \qquad \forall \; U\in\mcU \  .
\end{align}
Finally, for any given $X$, we define our estimator of the transformation by $\hat T(X):\R^d\to \mcU$. Accordingly, by leveraging~\eqref{eq:encoder}, our estimate $\hat Z(X)$ based on $X$ is given by
\begin{align}\label{eq:encoder2}
    \hat Z(X) = [\hat T(X)]^+ \cdot X \ .
\end{align}
\paragraph{Objective.}  Based on the discussion above, our objective is to find an estimator $\hat T(X): \R^d\to\mcU$ such that the estimate $\hat Z(X)$ formed by~\eqref{eq:encoder2} renders a valid DAG recovery and maintains scaling consistency, and when not possible, mixing consistency. Throughout the rest of the paper, when it is obvious from the context, we use the shorthand notations $\hat Z$ and $\hat T$ for $\hat Z(X)$ and $\hat T(X)$, respectively. 

\section{Score Functions and Their Properties}\label{sec:score-properties}

Score functions and their variations across different interventional environments play pivotal roles in our approach to identifying latent representations. In this section, we start by presenting the key properties for the score functions and a set of assumptions that would help us prove our identifiability results. Subsequently, we discuss which classes of latent causal models satisfy our set of assumptions.

\subsection{Assumptions and Properties}\label{sec:assumptions-properties}
First, we note that the identifiability of latent representations is viable only if each node is included in at least one intervention environment~\citep{seigal2022linear}, i.e., $\bigcup_{i\in[M]}I^m=[n]$. Hence, we focus on atomic interventions, which is a setting also adopted by ~\citep{seigal2022linear} and \citep{ahuja2022interventional}.
\begin{assumption}[Intervention Environments]\label{assumption:exhaustive_atomic}
We consider $M=n+1$ environments, consisting of one observational and $n$ atomic interventional ones such that each node will be intervened in exactly one environment, i.e., $\mcI = \{\emptyset,\{1\},\dots,\{n\}\}$. The index of the unknown environment in which node $i\in[n]$ is intervened is denoted by $m_i \in [n]$, i.e., $I^{m_i}=\{i\}$.
\end{assumption}
We investigate the variations of the score functions $\{s^{m}_{Z}(z):m\in[n]\}$ that are caused by the atomic interventions. The key insight is that an atomic intervention is expected to cause changes in only certain coordinates of the score function. Next, we present an assumption that ensures the identifiability of such changes in the score function.
\begin{assumption}[Interventional Regularity]\label{assumption:pq-parent-dependence}
For every node $i\in[n]$ we have
\begin{align}\label{eq:pq-parent-dependence}
    \frac{\partial }{\partial z_k}\frac{q_Z(z_i\mid z_{\Pa(i)})}{p_Z(z_i\mid z_{\Pa(i)})} \neq 0 \ , \qquad \forall k\in \Pa(i)\ . 
\end{align}
\end{assumption}
This assumption, in spirit, is similar to the interventional faithfulness assumption adopted in the causal learning literature \citep{yang2018characterizing,softunknown20}. In Section~\ref{sec:analyze-assumption} (Lemma~\ref{lemma:pq-parent-dependence}), we show that this assumption holds for widely-used canonical models. By invoking this assumption, the following lemma delineates the set of coordinates of the score function that are affected under different atomic interventions.

\begin{lemma}[Score Changes under Interventions]\label{lm:parent_change}
Under Assumption~\ref{assumption:pq-parent-dependence}, the observational score function $s_Z(z)$ and the interventional score function $s^{m}_{Z}(z)$ have distinct distributions in their $i$-th coordinate if and only if node $i$ or one of its children is intervened in $\mcE^{m}$, i.e., for $m\in[n]$
     \begin{align}
        \left[s_{Z}\right]_{i} {\overset{p_Z}{\neq}} \left[s_{Z}^{m} \right]_{i}\  \qquad \Longleftrightarrow \qquad i \in \overline{\Pa}(I^{m})\ . \\
     \end{align}
\end{lemma}
\begin{proof}
See Appendix~\ref{proof:parent_change}.
\end{proof}
Lemma~\ref{lm:parent_change} provides the necessary and sufficient conditions for the invariance of the coordinates of the score functions. To properly use Lemma~\ref{lm:parent_change}, we make a general assumption such that the effect of an intervention is not lost in any linear combination of the varying coordinates of the scores. To formalize this, for each node $i\in[n]$ we define 
\begin{align}
    \mcC_i \triangleq \{c\in\R^n\;:\; \exists j \in \overline{\Pa}(i) \;\; \mbox{such that }\; c_j\neq 0 \}\ . \label{eq:C}
\end{align}
\begin{assumption}\label{assumption:change} 
    Consider environment $\mcE^{m_i}$ in which node $i$ is intervened. Then, for all $c \in \mcC_i$ we have
    \begin{align}
            \left(c^{\top}\cdot s_{Z} \right) \overset{p_Z}{\neq} \left(c^{\top}\cdot s^{m_i}_{Z} \right) \ .       
            \label{eq:assumption:change-new:1}
    \end{align}
\end{assumption}
Next, we need an oracle to test whether a given coordinate of the score functions $s_{Z}$ and $s_{Z}^{m}$ at observational data from $\mcE^{0}$ are equal almost surely. 
\begin{assumption}[Testing Equivalence]\label{assumption:score_test}
For a given pair of environments $\mcE^{0}$ and $\mcE^{m}$, we have access to an oracle that samples $\hat z$ from $\hat Z(U,X)$ in the observational environment and determines if the following holds almost surely
\begin{align}
    \left[s_{\hat Z}(\hat z)\right]_i = \left[s_{\hat Z}^{m}(\hat z)\right]_i \ , \quad i \in [n] \ .
\end{align}
\end{assumption}
Note that we use this oracle to determine whether $[s_{\hat Z}]_i \overset{p_Z}{=} [s_{\hat Z}]_i$. In our analysis, we leverage the relationships among the score functions associated with the observed and latent variables. Through the following assumption, we ensure that the score functions $s_{X}$ and $\{s^{m}_{X}:m\in[n]\}$ can be computed.
\begin{assumption}[Score Function]\label{assumption:score-availability}
We can compute the score function $s_{X}(x)$ and $\{s^{m}_{X}:m\in[n]\}$ from samples of $X$ in the observational and interventional environments.
\end{assumption}
Next, we show how to compute $s_{Z}$ and $\{s^{m}_{Z}:m\in[n]\}$ from $s_{X}$ and $\{s^{m}_{X}:m\in[n]\}$, respectively, by leveraging the data generation process in~\eqref{eq:data-generation-process}. 
\begin{lemma}[Score Transformation]\label{lm:score-linear-transform}
    Consider random vectors $Y\in\R^r$ and $W\in\R^s$ that are related through $Y=AW$ such that $r \ge s$ and $A\in\R^{r\times s}$ is a full-rank matrix. The score functions of $Y$ and $W$, denoted by $s_Y$ and $s_W$, respectively, are related through
    \begin{align}\label{eq:score-linear-transform}
        s_{W}(w)
        &= A^{\top} s_{Y}(y) \ , \qquad \mbox{where} \qquad y=Aw\ .
    \end{align}
\end{lemma}
\begin{proof}
    See Appendix~\ref{proof:score-linear-transform}.
\end{proof}
By setting $W=Z$, $Y=X$, and $A=T$, from \eqref{eq:data-generation-process} and Lemma~\ref{lm:score-linear-transform} we have
\begin{align}
    s_{Z}(z) &= T^{\top} s_{X}(x) \ , &&  \mbox{where} \qquad x=Tz \label{eq:sz-from-sx1}\ , \\
    s_{Z}^{m}(z) & = T^{\top} s^{m}_{X}(x)\ , \quad \forall m \in [n] \ , &&  \mbox{where} \qquad x=Tz\ . \label{eq:sz-from-sx2}
\end{align}
Furthermore, we can also use Lemma~\ref{lm:score-linear-transform} to compute the score function of the auxiliary latent estimate $\hat Z(U,X)$ defined in \eqref{eq:encoder2}. By setting $W=\hat{Z}(U,X)$, $Y=X$, and $A=U$, from Lemma~\ref{lm:score-linear-transform} we have
\begin{align}\label{eq:sz-hat-from-sx1}
    s_{\hat{Z}}(\hat z) & = U^{\top} s_{X}(x) \ , && \mbox{where} \qquad x = U \hat z\ , \\
    s_{\hat{Z}}^{m}(\hat z) & = U^{\top} s^{m}_{X}(x) \ , \quad \forall m \in [n] \ , && \mbox{where} \qquad x = U \hat z \ . \label{eq:sz-hat-from-sx2}
\end{align}
Next, we show how to recover the true latent scores from the estimated latent scores. For this purpose for any candidate transformation $U\in\mcU$, we define $H(U)\in\R^{n\times n}$ as
\begin{align}
        H(U)  &\triangleq (T^{+} U)^{\top} \  . \label{eq:H-definition}
    \end{align}
We note that $H(U)$ is invertible for all $U \in \mcU$, formalized in the next lemma. 
\begin{lemma}\label{lm:H-definition-and-invertibility}
    For any candidate $U \in \mcU$, the scores of $\hat Z(U,X)$ and $Z$ are related through
    \begin{align}
        s_{\hat{Z}}(\hat z) &= H(U) \cdot  s_{Z}(z) \ , && \mbox{where} \qquad x = T z = U \hat{z} \label{eq:sz-hat-from-sz1}\ , \\
        s_{\hat{Z}}^{m}(\hat z) & = H(U) \cdot s^{m}_{Z}(z)\ , \quad \forall m \in [n] \ , && \mbox{where} \qquad x = T z = U \hat{z} \ . \label{eq:sz-hat-from-sz2}
    \end{align}
    Furthermore, $H(U)$ is invertible for all $U \in \mcU$.
\end{lemma}
\begin{proof}
    See Appendix~\ref{proof:H-definition-and-invertibility}.
\end{proof}
For any pair of matrices $A\in\R^{n\times n}$ and $A'\in\R^{n\times d}$, we define the \emph{change} matrices $\Delta_{X}(A')$, $\Delta_{Z}(A)$, $\Delta_{\hat Z}(A) \in \R^{n \times n}$ as follows.
\begin{align}
    [\Delta_X(A')]_{i,m} & \triangleq \mathds{1}\left([A'\cdot s_X]_i \overset{p_X}{\neq}  [A'\cdot s_X^m]_i \right) \ ,  \label{eq:def-delta-X} \\
    [\Delta_Z(A)]_{i,m} & \triangleq \mathds{1}\left([A\cdot s_Z]_i \overset{p_Z}{\neq}  [A\cdot s_Z^m]_i\right) \ , \label{eq:def-delta-Z} \\
    [\Delta_{\hat Z}(A)]_{i,m} & \triangleq \mathds{1}\left([A\cdot s_{\hat Z}]_i \overset{p_{\hat Z}}{\neq}  [A\cdot s_{\hat Z}^m]_i\right) \ . \label{eq:def-delta-Z-hat}
\end{align}
Note that these change matrices are related according to
\begin{align}
    \Delta_{Z}(A) &\overset{(\ref{eq:sz-from-sx1}, \ref{eq:sz-from-sx2})}{=} \Delta_{X}(A \cdot T^{\top}) \ . \label{eq:delta-z-from-delta-x} \\
    \Delta_{\hat Z}(A) &\overset{(\ref{eq:sz-hat-from-sx1}, \ref{eq:sz-hat-from-sx2})}{=} \Delta_{X}(A \cdot U^{\top}) \ , \label{eq:delta-z-hat-from-delta-x} \\
    \Delta_{\hat Z}(A) &\overset{(\ref{eq:sz-hat-from-sz1}, \ref{eq:sz-hat-from-sz2})}{=} \Delta_{Z}(A \cdot H(U)) \ . \label{eq:delta-z-hat-from-delta-z}
\end{align}

\subsection{Interpreting the Assumptions}\label{sec:analyze-assumption}
In this section, we assess the viability of Assumptions~\ref{assumption:exhaustive_atomic}--\ref{assumption:score-availability}. Our focus is primarily on Assumptions~\ref{assumption:pq-parent-dependence} and \ref{assumption:change} since Assumptions~\ref{assumption:exhaustive_atomic},~\ref{assumption:score_test}, and \ref{assumption:score-availability} are standard in causal representation learning literature and score-based methods and do not constrain the latent causal models or the intervention mechanisms. Specifically, every single node being intervened is known to be necessary for identifiability~\citep{seigal2022linear}, resulting in Assumption~\ref{assumption:exhaustive_atomic}, and Assumptions~\ref{assumption:score_test} and~\ref{assumption:score-availability} ensure that scores and their variations can be computed from data. 

The following lemma shows that Assumption~\ref{assumption:pq-parent-dependence} holds for a very general class of models, such as additive and multiplicative noise models. Furthermore, hard interventions ensure that Assumption~\ref{assumption:pq-parent-dependence} holds for any latent causal model. 
\begin{lemma}\label{lemma:pq-parent-dependence}
    Assumption~\ref{assumption:pq-parent-dependence} holds in any of the following settings.
    \begin{enumerate}
        \item All interventions are hard.
        \item Additive latent causal models, i.e., $z_i$ is related to $z_{\Pa(i)}$ via
        \begin{align}\label{eq:additive}
            Z_i = f_{p,i}(Z_{\Pa(i)}) + N_{p,i}\ , \quad \mbox{and} \quad  \quad Z_i = f_{q,i}(Z_{\Pa(i)}) + N_{q,i}\ ,
        \end{align}
        where $\{f_{p,i} : i \in [n]\}$ and $\{f_{q,i} : i \in [n]\}$ are observational and interventional causal mechanisms, respectively. The additive terms $\{N_{p,i} : i \in [n]\}$ and $\{N_{q,i} : i \in [n]\}$ account for noise and their pdfs are assumed to be analytic. 
        \item Multiplicative latent causal models, i.e., $Z_i$ is related to $Z_{\Pa(i)}$ via
        \begin{align}\label{eq:multiplicative}
            Z_i = f_{p,i}(Z_{\Pa(i)}) \cdot  N_{p,i}\ , \quad \mbox{and} \quad \quad Z_i = f_{q,i}(Z_{\Pa(i)}) \cdot N_{q,i}\ ,
        \end{align}
        where models and parameters follow the general models described for the additive model in \eqref{eq:additive}.
    \end{enumerate}
\end{lemma}
\begin{proof}
    See Appendix~\ref{proof:pq-parent-dependence}.
\end{proof}
Next, we establish the necessary and sufficient conditions under which Assumption~\ref{assumption:change} holds. Furthermore, we show that a large class of non-linear models in the latent space satisfy these conditions, including the two-layer neural networks. We have focused on such NNs since they effectively approximate continuous functions \citep{cybenko1989approximation}. Readily, the necessary and sufficient conditions can be investigated for other choices of non-linear functions.

\paragraph{Necessary and sufficient conditions.} Consider the canonical causal model with additive noise in~\eqref{eq:additive}, which is widely used in the literature on causal inference~\citep{peters2017elements} and causal representation learning~\citep{scholkopf2021toward}. The following theorem characterizes the necessary and sufficient conditions under which Assumption~\ref{assumption:change} is satisfied for the additive noise model in~\eqref{eq:additive}. In this subsection, we use  $\varphi$ as the shorthand for $z_{\Pa(i)}$. 
\begin{theorem}\label{theorem:condition-assumption}
For each node $i\in[n]$ consider the following two set of equations for $c\in\R^n$:
\begin{align}\label{eq:NS_equations}    
   \left\{
   \begin{array}{l}
         c_i - c^{\top} \cdot \nabla_z f_{p,i}(\varphi)  = 0  \\
         \\ 
         c_i - c^{\top} \cdot \nabla_z f_{q,i}(\varphi)  = 0
   \end{array}
   \right.\ , \qquad \qquad \forall \varphi \in\R^{|\Pa(i)|}\ .
\end{align}
Assumption~\ref{assumption:change} holds if and only if the only all solutions $c$ to \eqref{eq:NS_equations} satisfy $c\notin \mcC_i$, or based on~\eqref{eq:C}, equivalently
\begin{align}
    \forall j \in \overline{\Pa}(i): \quad  c_j= 0\ .
\end{align}
\end{theorem}
\begin{proof}
    See Appendix~\ref{proof:condition-assumption}.
\end{proof}
To provide some intuition about the conditions in Theorem~\ref{theorem:condition-assumption}, we consider a node $i \in [n]$ and discuss the conditions in the context of a few examples. Note that by sweeping $\varphi\in \R^{|\Pa(i)|}$ we generate a continuum of linear equations of the form: 
\begin{align}\label{eq:NS_equations2}
c_i - c^{\top} \cdot \nabla_z f_{p,i}(\varphi)  = 0\ , \qquad \mbox{and} \qquad c_i - c^{\top} \cdot \nabla_z f_{q,i}(\varphi)  = 0  \ .
\end{align}
Note that for all $j\not\in \Pa(i)$ we have $[\nabla_z f_{p,i}(\varphi)]_j=[\nabla_z  f_{q,i}(\varphi)]_j=0$. Hence, in finding the solutions to~\eqref{eq:NS_equations2} only the coordinates $\{j\in \overline{\Pa}(i)\}$ of $c$ are relevant. Let us define 
\begin{align} \label{eq:varphi}
  u_{p,i}(\varphi)\triangleq \nabla_{\varphi} f_{p,i}(\varphi) \ , \quad \mbox{and} \quad 
  u_{q,i}(\varphi)\triangleq \nabla_{\varphi} f_{q,i}(\varphi) \ ,
\end{align}
which are the gradients of $f_{p,i}$ and $f_{q,i}$ by considering only the coordinates of $z$ in $\{j\in {\Pa}(i)\}$. Accordingly, we also define $b$ by concatenating only the coordinates of $c$ with their indices in $\{j\in {\Pa}(i)\}$. Next, consider $w$ distinct choices of $\varphi$ and denote them by $\{\varphi^t\in\R^{|\Pa(i)|}:t\in[w]\}$. By concatenating the two equations in~\eqref{eq:varphi} specialized to these realizations, we get the following linear system with $2w$ equations and $|\Pa(i)|+1$ unknown variables.
\begin{align}
    \underset{\triangleq V_{2w\times (|\Pa(i)|+1)}}{\underbrace{
    \begin{bmatrix}
        [u_{p,i}(\varphi^1)]^{\top} & -1\\
        [u_{q,i}(\varphi^1)]^{\top} & -1\\
        \vdots & \vdots\\
        [u_{p,i}(\varphi^w)]^{\top} & -1\\
        [u_{q,i}(\varphi^w)]^{\top} & -1\\
    \end{bmatrix}}}
    \begin{bmatrix}
        b \\ c_i
    \end{bmatrix}
    =\boldsymbol{0}_{2w}\ .
\end{align}
When $V$ is full-rank, i.e., $\rank(V)=|\Pa(i)|+1$, the system has only the trivial solutions $c_i=0$ and $b=\boldsymbol{0}$. Then, we make the following observations.

\begin{enumerate}
    \item If $f_{p,i}$ and $f_{q,i}$ are linear functions, the vector spaces generated by $u_{p,i}$ and $u_{q,i}$ have dimensions 1. Subsequently, we always have $\rank(V)\leq 2$, rendering an underdetermined system when $|\Pa(i)|\geq 2$. Hence, when the maximum degree of $\mcG_Z$ is at least 2, a linear causal model does not satisfy Assumption~\ref{assumption:change}. Note that this observation is in line with the impossibility result of \cite{seigal2022linear} for the linear models under soft interventions.
    \item Similarly, if $f_{p,i}$ and $f_{q,i}$ are generalized linear functions, we always have $\rank(V)\leq 3$, and Assumption~\ref{assumption:change} is not satisfied when maximum degree of $\mcG_Z$ is at least $3$.
    \item If $f_{p,i}$ and $f_{q,i}$ are quadratic with full-rank matrices, i.e., $f_{p,i}(\varphi) = \varphi^{\top} A_p \varphi$ and $f_{q,i}(\varphi) = \varphi^{\top} A_q \varphi$ where $\rank(A_p)=\rank(A_q)=|\Pa(i)|$, there is a choice of $w \geq |\Pa(i)| + 1$ and realizations $\{\varphi^t \in \R^{|\Pa(i)|} : t \in [w] \}$ for which $\rank(V)=|\Pa(i)|+1$ and the system in~\eqref{eq:NS_equations2} admits only the trivial solutions $a_i=0$ and $b_i=\boldsymbol{0}$. Hence, quadratic causal models satisfy Assumption~\ref{assumption:change}.
    \item If $f_{p,i}$ and $f_{q,i}$ are two-layer NNs with a sufficiently large number of hidden neurons, they also render a fully determined system, and as a result, they satisfy Assumption~\ref{assumption:change}. 
\end{enumerate}
We investigate that last example in details as follows. Assume that $f_{p,i}$ and $f_{q,i}$ are two-layer NNs with $|\Pa(i)|$ inputs, $w_{p,i}$ and $w_{q,i}$ hidden nodes, respectively, and with sigmoid activation functions. Denote the weight matrices between input and hidden layers in $f_{p,i}$ and $f_{q,i}$ by $W^{p,i} \in \R^{w_{p,i} \times |\Pa(i)|}$ and $W^{q,i} \in \R^{w_{q,i} \times |\Pa(i)|}$, respectively. Furthermore, define $\nu^{p,i} \in \R^{w_{p,i}}$ and $\nu^{q,i} \in \R^{w_{q,i}}$  as the weights between the hidden layer and output in $f_{p,i}$ and $f_{q,i}$, respectively. Finally, define $\nu^{p,i}_0$ and $\nu^{q,i}_0$ as the bias terms. Hence, we have
\begin{align}
    f_{p,i}(\varphi) & = [\nu^{p,i}]^{\top} \sigma(W^{p,i} \varphi) + \nu^{p,i}_0  = \sum_{j=1}^{w_{p,i}} \nu^{p,i}_j \sigma(W^{p,i}_j \varphi) + \nu^{p,i}_0  \ , \label{eq:nn-expression} \\
    f_{q,i}(\varphi) & = [\nu^{q,i}]^{\top} \sigma(W^{q,i} \varphi) + \nu^{q,i}_0  = \sum_{j=1}^{w_{q,i}} \nu^{q,i}_j \sigma(W^{q,i}_j \varphi) + \nu^{q,i}_0  \ , \label{eq:nn-expression2} 
\end{align}
in which activation function $\sigma$ is applied element-wise and $A_j$ denotes the $j$-th row of a matrix $A$.
\begin{lemma}\label{prop:nn}
Consider NNs $f_{p,i}$ and $f_{q,i}$ specified in~\eqref{eq:nn-expression} and \eqref{eq:nn-expression2}. If for all $i \in[n]$ we have
   \begin{align}
       \max\{\rank(W^{p,i})\; , \; \rank(W^{q,i})\} = |\Pa(i)|\ ,
   \end{align}
  then Assumption~\ref{assumption:change} holds.
\end{lemma}
\begin{proof}
    See Appendix~\ref{proof:nn}.
\end{proof}

\section{SCALE-I Algorithm}\label{sec:algorithm}

\begin{algorithm}[t]
\caption{{\bf S}core-based {\bf Ca}usal {\bf L}atent {\bf E}stimation via {\bf I}nterventions (SCALE-I)}
\label{alg:main}
\begin{algorithmic}[1]
\State \textbf{Input:} $\mcU$, observational environment $\mcE^{0}$ and interventional environments $\mcE: \{\mcE^{m} \; : m \in [n]\}$ 
\State \textbf{Step 1:} Compute observational scores: $s_{X}$ and $\{s^{m}_{X} : m \in [n]\}$.
\State \textbf{Step 2:} Minimize score variations:
$$\mcU_1 \triangleq \left\{V \in \mcU : \norm{\Delta_{X}(V^{\top})}_0 = \min_{U\in\mcU}\norm{\Delta_{X}(U^{\top})}_0\right\}$$
\State \textbf{Step 3:} Refine transformation candidates
$$\mcU_{\rm S} \triangleq \{ U \in \mcU_{1}\; : \;  \exists P \in \Pi \mbox{ such that } \; \Delta_{X}(U^{\top})P \; \mbox{  is upper triangular} \}  $$
\State For each $U\in\mcU_{\rm S}$, $P_2(U)$ denotes the permutation matrix that transforms $\Delta_{X}(U^{\top})$ to upper triangular.
\State \textbf{Step 4:} transformation estimate under soft intervention: Randomly select a $\hat T \in \mcU_{\rm S}$
\State Construct the upper triangular matrix $ K(\hat T) \triangleq \Delta_{X}(\hat{T}^{\top})P_2(\hat{T})$.
\State Construct latent DAG $\mcG_{\hat Z}$ such that $\mbox{for all nodes $i\in\mcG_{\hat Z}$: } \quad \Pa(i) \triangleq \{j : [K(\hat T)]_{j,i} = 1\}$
\State Construct set of surrounded nodes: $\mcS_{\pi} \triangleq \{\pi_i : \sur(\pi_i) \neq \emptyset \}$
\If {the interventions are hard}
\State \textbf{Step 5:} Ensure using the same DAG as in Step 4:
$$ \mcU_{2} \triangleq \{U \in \mcU_{\rm S} : \; \Delta_{X}(U^{\top}) = \Delta_{X}(\hat T^{\top})\} \ $$
\State Require independencies between surrounded nodes and their parents 
$$\mcU_{\rm H} \triangleq \{U \in \mcU_{2} : \exists \;  \mcE^{m} \in \mcE \;\; \hat Z^{m}_{i}(U,X) \ci \hat Z^{m}_{j}(U,X) \; , \; \forall i \in \mcS_{\pi} \ , j \in \sur(i) \} \  $$
\State Update the choice of estimator: randomly select a transformation $\hat{T}$ within $\mcU_{\rm H}$
\EndIf \\
\Return  $\hat{Z} = \hat{T}^+ \cdot X $
\end{algorithmic}
\end{algorithm}%

In this section, we provide {\bf S}core-based {\bf Ca}usal {\bf L}atent {\bf E}stimation via {\bf I}nterventions (SCALE-I) algorithm. Under Assumption \ref{assumption:exhaustive_atomic}, our inputs are observational environment $\mcE^{0}$ and interventional environments $\{\mcE^{m} : m \in [n]\}$, along with set of candidate transformations $\mcU$. The algorithm consists of six steps, and summarized in Algorithm \ref{alg:main}. We describe the steps involved next.

\begin{description}
    \item [Step 1 -- Observational scores:] We start by computing observational scores $s_{X}$ and $\{s^{m}_{X} : m \in [n]\}$ under all possible environments. 
    
    \item [Step 2 -- Minimize score variations:] First, we note that the coordinates of the score changes in $s_Z$ capture the relations in latent DAG via Lemma \ref{lm:parent_change}. Hence, it is critical to obtain the true score changes to recover latent DAG.
        Next, we can compute the scores of the estimated latent variables $s_{\hat Z}$ and $\{s^{m}_{\hat Z} : m \in [n]\}$ for any candidate transformation $U\in\mcU$ by using observational scores and leveraging Lemma~\ref{lm:score-linear-transform}. The key property in this step is that variations of the estimated score under an incorrect transformation $U$ will be more than the variations of the true scores. 
        Based on this, we refine the set $\mcU$ to $\mcU_1\subseteq \mcU$ as 
        \begin{equation}
            \mcU_1 \triangleq \left\{V \in \mcU : \norm{\Delta_{X}(V^{\top})}_0 = \min_{U\in\mcU}\norm{\Delta_{X}(U^{\top})}_0\right\} \ .  \label{eq:algo-mcu1-def}
        \end{equation}
        In Lemma \ref{lm:min-sparsity}, we prove that score change matrix of a $U\in\mcU_1$ is \emph{equal to a row permutation of the true score change matrix}. If this row permutation has a valid causal order, we recover the DAG. However, we need one more step to ensure that.

    \item [Step 3 -- Refine transformation candidates:] 
    Minimizing the score variations in Step~2 does not necessarily rule out the invalid permutations of the latent estimates. To circumvent this issue, we make the \emph{observation} that by permuting the columns of $\Delta_{X}(T^{\top})$, we can reach an upper triangular matrix, which we will formalize in~\eqref{eq:D}. By invoking this algebraic structure, we further refine $\mcU_1$ to generate $\mcU_{\rm S}$ as follows
        \begin{align}
            \mcU_{\rm S} \triangleq \{ &U \in \mcU_{1}\; : \;  \exists P \in \Pi \mbox{ such that } \; \Delta_{X}(U^{\top}) P \; \mbox{  is upper triangular} \}  \ .  \label{eq:algo-mcus-def}
        \end{align}
        For each $U\in\mcU_{\rm S}$, we denote the permutation matrix that transforms $\Delta_{X}(U^{\top})$ to an upper triangular matrix by $P_2(U)$. Next step shows how $P_2(U)$ helps recovering $\mcG_{Z}$.
        
   \item [Step 4 -- Transformation estimate and latent DAG recovery under soft intervention:]  
Given $\mcU_{\rm S}$ generated in Step~3, we randomly select a $\hat T \in \mcU_{\rm S}$ as our estimate of $T$ and construct the upper triangular matrix 
        \begin{align}\label{eq:def-K-hat-T}
        K(\hat T) \triangleq \Delta_{X}(\hat{T}^{\top}) \cdot P_2(\hat{T}) \ .   
        \end{align}
        We can construct a new DAG $\mcG_{\hat Z}$ from $K(\hat T)$ as follows. We create $n$ nodes and assign the non-zero coordinates of the $i$-th column of $K(\hat T)$ as the parents of node $i$ in $\mcG_{\hat Z}$, i.e., 
        \begin{align}\label{eq:DAG-construction}
            \Pa(i) \triangleq \{j : [K(\hat T)]_{j,i} = 1\}\ , \quad  \forall i \in\mcG_{\hat Z} \ .
        \end{align}
        We denote the permutation that maps the nodes of $\mcG_{Z}$ to those of $\mcG_{\hat Z}$ by $\pi$. Subsequently, from $\mcG_{\hat Z}$, we construct the set of surrounded nodes according to Definition~\ref{def:surrounded} as
        \begin{align}\label{eq:surrounded-pi}
          \mcS_{\pi} \triangleq \{\pi_i : \sur(\pi_i) \neq \emptyset\}\ .
        \end{align}
        In Theorem \ref{th:soft-final-results}, we prove that $\mcG_{\hat Z}$ has the same structure as $\mcG_{Z}$ under permutation $\pi$, and $\pi$ is a valid causal order. Furthermore, we prove that only the surrounded nodes have mixing in their estimations, and all non-surrounded nodes are estimated correctly up to scaling.
    
    \item [Step 5 -- Transformation estimate under hard intervention:]
    We will have this additional step only in the case of hard interventions to further refine our set of transformation candidates. We have already estimated a DAG in Step 4 based on a $\hat T$ randomly chosen from $\mcU_{\rm S}$. Note that a different choice of $\hat T$ may have led to a different graph $\mcG_{\hat Z}$. Before applying the conditions associated with hard interventions, we first take the subset of $\mcU_{\rm S}$ whose elements lead to the same estimated DAG produced in Step~4 based on the selected $\hat T$. In Proposition~\ref{prop:same-dag}, we show that the elements of the following set satisfy this condition:    
    \begin{equation}
        \mcU_{2} \triangleq \{U \in \mcU_{\rm S} : \; \Delta_{X}(U^{\top}) = \Delta_{X}(\hat T^{\top})\} \ . \label{eq:algo-mcu2-def}
    \end{equation}    
    Then, note that a hard intervention on a node breaks its dependence on its parents. Since all atomic interventions are available, we enforce the independence relations between estimates $[\hat Z^{m}(U,X)]_i$ and $[\hat Z^{m}(U,X)]_j$ to hold in at least one environment $m \in [n]$: 
    \begin{equation}
        \mcU_{\rm H} \triangleq \{U \in \mcU_{2} : \exists \;  \mcE^{m} \in \mcE \;\; [\hat Z^{m}(U,X)]_{i} \ci [\hat Z^{m}(U,X)]_{j} \; , \; \forall i \in \mcS_{\pi} \ , j \in \sur(i) \} \ .  \label{eq:algo-mcuh-def}
    \end{equation}
   This step is concluded by updating our estimator and randomly selecting $\hat T \in \mcU_{\rm H}$.

    \item [Step 6 -- Latent estimates:] At the end, the latent variables are estimated using the final $\hat T$ according to $\hat Z = \hat{T}^{+} X$, as specified in~\eqref{eq:encoder2}.
\end{description}
The detailed steps involved are summarized in Algorithm~\ref{alg:main}.

\section{Identifiability Results}\label{sec:main}
In this section, we provide the identifiability results for the estimates generated by SCALE-I algorithm. Our primary focus will be on the soft interventions, analyzed in Section~\ref{sec:main_soft}. The results will be extended to the setting with hard interventions, which are presented in Section~\ref{sec:main_hard}. 

\subsection{Identification of Latent Variables with Soft Interventions using Scores}\label{sec:main_soft}

\paragraph{Minimal score variations (Step 2).} This step of the algorithm is designed based on the key property that if we use an incorrect transformation $U$, the changes in the scores of estimated latent variables will not be less than the number of variations of the true scores. 
We formally prove this property in Lemma~\ref{lm:min-sparsity}. For this purpose, we provide the following proposition.
\begin{proposition}\label{fact-permutation-diagonal}
If $A \in \R^{n \times n}$ is a full-rank matrix, then there exists a permutation matrix $P\in\Pi$ such that the diagonal elements of $P A$ are non-zero. 
\end{proposition}
\begin{proof}
    See Appendix~\ref{proof:fact-permutation-diagonal}.
\end{proof}
Using Proposition \ref{fact-permutation-diagonal}, for each $U \in \mcU$, there exists a permutation matrix $P_1(U) \in\Pi $ such that $P_1(U) \cdot H(U)$ has non-zero diagonal entries. Accordingly, we define
\begin{align}\label{eq:def-H-bar}
    \bar H(U) &\triangleq P_1(U) \cdot H(U) \ .
\end{align}
\begin{lemma}\label{lm:min-sparsity}
For every $U \in \mcU$, the score change matrix $\Delta_{\hat Z}(I_n)$ associated with the estimated latent variables is at least as dense as the score change matrix $\Delta_{Z}(I_n)$ associated with the true latent variables. In other words,
\begin{align}
    \norm{\Delta_{\hat Z}(I_n)}_0 \geq \norm{\Delta_{Z}(I_n)}_0 \ . \label{eq:lm-8-total-statement}
\end{align}
Furthermore, \eqref{eq:lm-8-total-statement} holds with equality only if
\begin{align}
    P_1(U) \cdot \Delta_{\hat Z}(I_n) &= \Delta_{Z}(I_n) \ , \label{eq:lm:min-sparsity-p-outside}
\end{align}
where $P_1(U) \in \Pi$ is uniquely specified by $U$.
\end{lemma}
\begin{proof}
See Appendix \ref{appendix:lm:min-sparsity}.
\end{proof}
Note that $\Delta_{Z}(I_n)$ is the true score change matrix under the true transformation $T$. Hence, the inequality in \eqref{eq:lm-8-total-statement} shows that $T$ belongs to the set of candidates $U \in \mcU$ that minimize $\norm{\Delta_{\hat Z}(I_n)}_0$. By invoking the condition, we can shrink the space of transformation candidates as follows 
\begin{align}
    \mcU_1 &= \left\{V \in \mcU : \norm{\Delta_{\hat Z(V,X)}(I_n)}_0 = \min_{U\in\mcU}\norm{\Delta_{\hat Z(U,X)}(I_n)}_0\right\} \label{eq:U1-Z-hat} .    
\end{align}
Note that, by \eqref{eq:delta-z-hat-from-delta-x}, we have $\Delta_{\hat Z(U,X)}(I_n) = \Delta_{X} (U^{\top})$, and $\mcU_1$ defined in \eqref{eq:U1-Z-hat} is equivalent to the definition in \eqref{eq:algo-mcu1-def}. Next, we investigate other properties of the candidate transformations in $\mcU_1$ that can help to further refine $\mcU_1$. Recall that, for any transformation $U \in \mcU$, we have
\begin{align}
    \hat Z(U,X) &\overset{\eqref{eq:encoder}}{=} U^{+} X \overset{\eqref{eq:data-generation-process}}{=} U^{+} T Z \ .
\end{align}
By using $U^{+} T = [H(U)]^{-\top}$ from \eqref{eq:H-definition}, we have
\begin{align}
    \hat Z(U,X) &= [H(U)]^{-\top} \cdot Z \ . \label{eq:z-hat-from-z} 
\end{align}
Thus, in order to investigate the properties of estimate $\hat Z(U,X)$ for a transformation $U \in \mcU_1$, we instead investigate the non-zero coordinates of $H(U)$ and $[H(U)]^{-\top}$.

\begin{lemma}\label{lm:rho}
For $U \in \mcU_{1}$, the matrix $\bar H(U)$ defined in \eqref{eq:def-H-bar} satisfies
\begin{align}
    I_n \preccurlyeq \mathds{1}(\bar H(U)) &\preccurlyeq I_n + \Sigma  \ ,  \quad \Sigma \in \{0,1\}^{n \times n} \ , \label{eq:barH}
\end{align}
where
\begin{align}
    \Sigma_{i,j} &\triangleq \begin{cases}
        0 \ , &\textnormal{if} \quad \overline{\Ch}(j) \not \subseteq \Ch(i) \ , \\
        1 \ , &\textnormal{otherwise}
    \end{cases} \ . \label{eq:rho}
\end{align}
Furthermore, $\bar H(U)$ also satisfies $I_n \preccurlyeq \mathds{1}([\bar H(U)]^{-\top}) \preccurlyeq I_n + \Sigma^{\top}$.
\end{lemma}
\begin{proof}
    See Appendix~\ref{proof:rho}.
\end{proof}
Next, observe that the true score change matrix $\Delta_{Z}(I_n)$ captures the latent DAG due to Lemma~\ref{lm:parent_change} (discussed in Step~3). Hence, \eqref{eq:lm:min-sparsity-p-outside} effectively means that if row permutation $[P_1(U)]^{\top}$ corresponds to a valid causal order, we can obtain the latent DAG. The next step addresses this requirement.

\paragraph{Refine transformation candidates (Step 3).}
We start by further scrutinizing $\Delta_{Z}(I_n)$. First, we reorganize the columns of $\Delta_{Z}(I_n)$ with respect to the intervention order $(m_1,\dots,m_n)$, and define 
\begin{align}\label{eq:D}
    D \triangleq \left[ \Delta_{Z}^{m_1}(I_n) \ , \dots, \ \Delta_{Z}^{m_n}(I_n) \right] \ .
\end{align}
If $j > i$, node $j$ is not in $\overline{\Pa}(i)$ and by Lemma \ref{lm:parent_change} we have $[\Delta_{Z}(I_n)]_{j,m_i}=D_{j,i}=0$, and $D$ is upper triangular. Next, we select the permutation matrix $P_{\mcI} \in \Pi$ such that
\begin{align}\label{eq:def-pi-I}
    [P_{\mcI}]_{i,m_i} = 1 \ , \quad  \forall i \in [n] \ .
\end{align}
Apparently, $\Delta_{Z}(I_n)$ is a column permutation of $D$:
\begin{align}
    \Delta_{Z}(I_n) = \left[ \Delta_{Z}^{1}(I_n) \ , \dots, \ \Delta_{Z}^{n}(I_n) \right] = D \cdot P_{\mcI} \ . \label{eq:def-D-prime}
\end{align}
Therefore, we can require $U$ to satisfy the condition that the columns of $\Delta_{\hat Z(U,X)}(I_n)$ can be permuted to an upper triangular matrix $D$, and refine the set $\mcU_1$ to $\mcU_{\rm S}$ as specified in \eqref{eq:algo-mcus-def}. 

For each $U\in\mcU_{\rm S}$, we denote the permutation matrix that transforms $\Delta_{X}(U^{\top})$ to an upper triangular matrix by $P_2(U)$, and denote the resulting matrix by
\begin{align}
    K(U) &\triangleq \Delta_{X}(U^{\top}) \cdot P_2(U) \ . \label{eq:def-K}
\end{align}

\paragraph{Transformation estimate and latent DAG recovery under soft intervention (Step 4).}
Given $\mcU_{\rm S}$, we randomly select $\hat T \in \mcU_{\rm S}$ as our transformation estimate under soft interventions and compute $K(\hat T)$. We prove that the construction of $\mcU_{\rm S}$ ensures that $[P_1(\hat T)]^{\top}$ corresponds to a valid causal order. Subsequently, this property of $[P_1(\hat T)]^{\top}$ leads to the recovery of the latent DAG. Furthermore, we show that this choice of $\hat T \in \mcU_{\rm S}$ leads to an estimate $\hat Z(X)$ that satisfies mixing consistency.

\begin{theorem}[Identifiability under Soft Intervention]\label{th:soft-final-results}
Under Assumptions A$_{\ref{assumption:exhaustive_atomic}}$ -- A$_{\ref{assumption:score-availability}}$ and soft interventions, the final estimate $\hat T(X)$, which is chosen randomly from $\mcU_{\rm S}$, satisfies the following properties:
\begin{enumerate}
    \item The estimate DAG $\mcG_{\hat Z}$ is equal to the true DAG $\mcG_{Z}$ under a valid causal order $\pi$ that maps the nodes of $G_{Z}$ to those of $\mcG_{\hat Z}$ and $\hat Z(X)$ factorizes with respect to $\mcG_{\hat Z}$. 
    \item Estimate $\hat Z(X)$ satisfies mixing consistency. Formally,
        \begin{align}
        \hat Z^t= P_\pi\cdot (C+B)\cdot Z^t \ , \label{eq:soft-final-results}
    \end{align}
    where $C\in\R^{n\times n}$ is a constant \emph{diagonal} matrix and $B\in\R^{n\times n}$ is a constant \emph{sparse} matrix that accounts for mixing and it satisfies 
    \begin{align}
        j\notin {\rm sur}(i) \quad \Rightarrow \quad B_{i,j}= 0\ , \qquad \forall i\in[n]\ . \label{eq:soft-final-results-B-def}
    \end{align}
\end{enumerate}
\end{theorem}
\begin{proof}
See Appendix~\ref{proof:soft-final-results}.
\end{proof}

\subsection{Identification of Latent Variables with Hard interventions}\label{sec:main_hard}
Via hard interventions, we can improve our identifiability results and resolve the mixing issue for surrounded variables.

\paragraph{Transformation estimate under hard intervention (Step 5).}
We selected an estimate $\hat T \in \mcU_{\rm S}$ in Step~4 and formed an estimate latent DAG $\mcG_{\hat Z}$. Note that, choosing a different $\hat T \in \mcU_{\rm S}$ could have led to a different DAG $\mcG_{\hat Z}$. In this step, we aim to work with the same $\mcG_{\hat Z}$. The following proposition shows that the set $\mcU_{2}$ specified in \eqref{eq:algo-mcu2-def} enforces this condition.
\begin{proposition}\label{prop:same-dag}
    For any $U \in \mcU_{2}$, we have $K(U)=K(\hat T)$, and the DAGs constructed from $K(U)$ and $K(\hat T)$ are the same.
\end{proposition}
\begin{proof}
    See Appendix~\ref{proof:same-dag}.
\end{proof}
Next, we use the properties of hard interventions. By definition, hard interventions are a subset of soft interventions in which the causal mechanism of an intervened node $i$ loses its functional dependence on $\Pa(i)$. The next statement is a direct result of this additional property, which is exclusive to hard interventions.

\begin{proposition}\label{fact:hard-indep}
For the environment $\mcE^{m_i}$ in which node $i$ is hard intervened, we have
\begin{align}
    Z_{i}^{m_i} \ci Z_{\nondesc(i)}^{m_i}\,,
\end{align}
where $\nondesc(i)$ is the set of non-descendants of node $i$ in the original graph $\mcG_{Z}$.
\end{proposition}
\begin{proof}
Each variable in a DAG is independent of its non-descendants given its parents. When $i$ is hard intervened, it has no parents and the statement follows directly.
\end{proof}
The key idea here is to constrain the estimated latent variables to conform to Proposition \ref{fact:hard-indep}. Note that we only need to worry about the surrounded nodes $\mcS_{\pi}$ in \eqref{eq:surrounded-pi}. Hence, for $i \in \mcS_{\pi}$ and $j \in \sur(i)$, we enforce independence between $[\hat Z^{m}(U,X)]_i$ and $[\hat Z^{m}(U,X)]_j$ on at least one environment $m \in [n]$. Formally, we refine $\mcU_{2}$ to $\mcU_{\rm H}$ as specified in \eqref{eq:algo-mcuh-def}.
Finally, we update our choice of transformation under hard interventions by randomly selecting a $\hat T \in \mcU_{\rm H}$.

\begin{theorem}[Identifiability under Hard Intervention]\label{th:hard-final-results}
Under Assumptions A$_{\ref{assumption:exhaustive_atomic}}$ -- A$_{\ref{assumption:score-availability}}$ and hard interventions, the latent variables are identified up to scaling consistency under a valid causal order. Formally, for any realization of $Z$, we have
\begin{align}
    \hat Z^t = P_{\pi}\cdot C\cdot Z^t \ , \label{eq:hard-final-results}
\end{align}
where $C \in \R^{n\times n}$ is a constant diagonal matrix and $P_{\pi}$ corresponds to a valid causal order $\pi$.
\end{theorem}
\begin{proof}
See Appendix~\ref{proof:hard-final-results}.
\end{proof}


\section{Conclusion and Future Work}
We have established identifiability results for latent causal representations from interventional data under an unknown but fixed linear transformation that relates the latent causal variables and observed samples. Under single-node stochastic hard interventions that cover every node in the latent space, we have shown the identifiability of the transformation (up to scaling and permutations with valid causal orders). For soft interventions, we recover the causal DAG among the true latent variables and show that the recovered latent variables are Markov with respect to the recovered DAG.

There are several interesting directions for future work. First, extending our score-based approach to non-linear transformations can help address practical problems, e.g., identifying the latent representations of real-world image data. However, we note that this would reinforce an existing result by \citep{ahuja2022interventional} that points in this direction. Second, extending single-node intervention setting to multi-node interventions can reduce the total number of required environments. Third, we plan to investigate stochastic transformations instead of the deterministic ones considered in this paper. Finally, investigating the identifiability results for causally insufficient graphs in the latent space can also be a promising direction.

\newpage
\appendix

\section{Proofs of Score Function Properties and Transformations}

\subsection{Proof of Lemma~\ref{lm:parent_change}}\label{proof:parent_change}

We give the proof for soft interventions. The proof for hard interventions follows from similar arguments. 

\paragraph{Proof of $\left[s_{Z}\right]_{i} {\overset{p_Z}{\neq}} \left[s_{Z}^{m} \right]_{i}\ \implies i \in \overline{{\rm pa}}(I^{m})$:}
Since we consider atomic interventions by Assumption \ref{assumption:exhaustive_atomic}, let us denote the node intervened in $\mcE^{m}$ by $k$, i.e.\ $I^{m}=\{k\}$. Following \eqref{eq:sz_decompose} and \eqref{eq:sz_m_decompose}, the latent scores $s_{Z}(z)$ and $s_{Z}^{m}(z)$ simplify to 
\begin{align}
    s_{Z}(z) &= \nabla_{z} \log p_{Z}(z_k \med z_{\Pa(k)}) + \sum_{j \in [n],\ j \neq k} \nabla_{z} \log p_{Z}(z_j \med z_{\Pa(j)}) \ , \label{eq:s_z_decompose_obs} \\
    s_{Z}^{m}(z) &= \nabla_{z} \log q_{Z}(z_k \med z_{\Pa(k)}) + \sum_{j \in [n],\ j \neq k} \nabla_{z} \log p_{Z}(z_j \med z_{\Pa(j)}) \ . \label{eq:s_z_decompose_int}
\end{align}
Hence, $s_{Z}(z)$ and $s_{Z}^{m}(z)$ differ in only the causal mechanism of node $k$. Next, we check the derivatives of $p_{Z}(z_k \med z_{\Pa(k)})$ and $q_{Z}(z_k \med z_{\Pa(k)})$ in their $i$-th coordinates. Note that these two depend on $Z$ only through $\{ Z_{i} : \; i \in \overline{\Pa}(k) \}$. Therefore, if $i \notin \overline{\Pa}(k)$
\begin{align}
    \pdv{z_{i}} \log p_{Z}(z_k \med z_{\Pa(k)}) = \pdv{z_{i}} \log {q}_{Z}(z_k \med z_{\Pa(k)}) = 0 \ , \label{eq:s_z_zero_indices}
\end{align}
which indicates that if $i \notin \overline{\Pa}(k)$, then $\left[s_{Z}\right]_{i} \overset{p_Z}{=} \left[s_{Z}^{m}\right]_{i}$. This, equivalently, means that if $ \left[s_{Z}\right]_{i} \overset{p_Z}{\neq} \left[s_{Z}^{m}\right]_{i}$, then $i \in \overline{\Pa}(I^{m})$. 

\paragraph{Proof of $\left[s_{Z}\right]_{i} {\overset{p_Z}{\neq}} \left[s_{Z}^{m} \right]_{i}\ \impliedby i \in \overline{{\rm pa}}(I^{m})$:} 
Note that the two score functions $s_Z$ and $s_Z^m$ are equal in their coordinate $i\in\overline{\rm pa}(k)$ only if
\begin{align}
    \frac{\partial \log p_Z(z_k  \mid z_{{\rm pa}(k)})}{\partial z_i}  =
    \frac{\partial \log q_Z(z_k  \mid z_{{\rm pa}(k)})}{\partial z_i} \ . \label{eq:p_i_q_i} 
\end{align} 
Equivalently,
\begin{align}
    0 &= \frac{\partial \log q_Z(z_k  \mid z_{{\rm pa}(k)})}{\partial z_i} -
    \frac{\partial \log p_Z(z_k  \mid z_{{\rm pa}(k)})}{\partial z_i} \\
    &= \frac{\partial}{\partial z_i} \log \frac{q_Z(z_k \mid z_{\Pa(k)})}{p_Z(z_k \mid z_{\Pa(k)})} \ . \label{eq:derivate-pq-parent-step}
\end{align} 
However, Assumption~\ref{assumption:pq-parent-dependence} ensures that the ratio in \eqref{eq:derivate-pq-parent-step} varies with $z_i$ and implies that its derivative with respect to $z_i$ is non-zero, which is a contradiction, and the proof is complete.

\subsection{Proof of Lemma \ref{lm:score-linear-transform}}\label{proof:score-linear-transform}
Instead of $Y = AW$, we start with a more generic case where the data generation mechanism is given by $Y = g(W)$, where $g: \; \R^{s} \to \R^{r}$ is a differentiable and injective function. The realizations of $W$ and $Y$ are related through $y = g(w)$. Denote the Jacobian matrix of $g$ at point $w \in \R^{s}$ by $J_{g}(w)$, which is an $r \times s$ matrix with entries given by
\begin{equation}
    [J_{g}(w)]_{i,j} = \pdv{[g(w)]_{i}}{w_{j}}(x) = \pdv{y_i}{w_j} \ , \quad \forall i \in [r] \; , j \in [s] \ . \label{eq:g-jacobian-defn}
\end{equation}
In this case, the pdfs of $W$ and $Y$ are related through \citep{boothby2003introduction}
\begin{equation}
    p_{W}(w) = p_{Y}(y) \cdot \left|\det([J_{g}(w)]^{\top} \cdot J_{g}(w))\right|^{1/2} \ . \label{eq:px-py-via-g-jacobian}
\end{equation}
Next, note that the gradient of a generic differentiable function $f: \; \R^{r} \to \R^{r}$ with respect to $w \in \R^{s}$ is given by
\begin{align}
    [\nabla_{w} f(y)]_{i} &= \pdv{w_{i}} f(y) = \sum_{j = 1}^{r} \pdv{f(y)}{y_{j}} \cdot \pdv{y_{j}}{w_{i}} = \sum_{j = 1}^{r} [\nabla_{y} f(y)]_{j} \cdot [J_{g}(w)]_{j,i} \ . \label{eq:lm-score-tr-tmp1}
\end{align}
Hence, from \eqref{eq:lm-score-tr-tmp1} and $y = g(w)$, more compactly, we have
\begin{align}
    \nabla_{w} f(y) &= [J_{g}(w)]^{\top} \cdot \nabla_{y} f(y) \ . \label{eq:nabla-x-fy}
\end{align}
Next, given the identities in \eqref{eq:px-py-via-g-jacobian} and \eqref{eq:nabla-x-fy}, we find the relationship between score functions of $W$ and $Y$.
\begin{align}
    s_{W}(w)
    &= \nabla_{w} \log p_{W}(w) \label{eq:score-transform-proof-tmp1} \\
    \overset{\eqref{eq:px-py-via-g-jacobian}}&{=} \nabla_{w} \log p_{Y}(y) + \nabla_{w} \log \left|\det([J_{g}(w)]^{\top} \cdot J_{g}(w))\right|^{1/2} \\ 
    \overset{\eqref{eq:nabla-x-fy}}&{=} [J_{g}(w)]^{\top} \cdot \nabla_{y} \log p_{Y}(y) + \nabla_{w} \log \left|\det([J_{g}(w)]^{\top} \cdot J_{g}(w))\right|^{1/2} \\ 
    &= [J_{g}(w)]^{\top} \cdot s_{Y}(y) + \nabla_{w} \log \left|\det([J_{g}(w)]^{\top} \cdot J_{g}(w))\right|^{1/2} \ . \label{eq:sx-sz-for-g}
\end{align}
Now, invoke $Y = g(W) = AW$. The Jacobian matrix in this case becomes $J_{g}(w) = A$, which is independent of $w$. Hence, the relationship in \eqref{eq:score-transform-proof-tmp1}-\eqref{eq:sx-sz-for-g} reduces to $s_{W}(w) = A^{\top} s_{Y}(y)$. 

\subsection{Proof of Lemma~\ref{lm:H-definition-and-invertibility}}\label{proof:H-definition-and-invertibility}
First, note that based on the relationship in \eqref{eq:data-generation-process} and the definition of $\mcU$ in \eqref{eq:reconstruction-prop}, for all $U \in \mcU$, we have
\begin{align}
    x &= U U^{+} x \ , \quad \forall x \in \image(T) \ ,
\end{align}
which is equivalent to
\begin{align}
    T z &= U U^{+} T z \ , \quad \forall z \in \R^{n} \ . \label{eq:t-hat-t-x}
\end{align}
The identity in \eqref{eq:t-hat-t-x} implies
\begin{align}
    \image(T) = \image(U U^{+} T) \ . \label{eq:im-t-eq-im-u-ui-t}
\end{align}
We can construct the set $\image(U U^{+} T)$ as the image of the set $\image(U^{+} T)$ formed by linear transformation via $U$. Since $U^{+} T \in \R^{n \times n}$, we get $\image(U^{+} T) \subseteq \R^{n}$. This implies
\begin{align}
    \image(U U^{+} T) \subseteq \image(U) \ .
\end{align}
Therefore, from \eqref{eq:im-t-eq-im-u-ui-t} we get
\begin{align}
    \image(T) \subseteq \image(U) \ . \label{eq:H-def-T-U-image-subseteq}
\end{align}
Next we show that \eqref{eq:H-def-T-U-image-subseteq} holds with equality using the property $\rank(T) = \rank(U) = n$. Rank of a matrix $A$ is equal to the dimension of its column space $\image(A)$. Therefore, $\rank(T) = \rank(U)$ is equivalent to $\dim(\image(T)) = \dim(\image(U))$. Since $\image(T)$ and $\image(U)$ are subspaces of $\R^{n}$ and \eqref{eq:H-def-T-U-image-subseteq} states that $\image(U)$ includes $\image(T)$, $\dim(\image(T)) = \dim(\image(U))$ is satisfied if and only if
\begin{align}
    \image(T) = \image(U) \ . \label{eq:t-u-image-equality}
\end{align}
Next, we prove \eqref{eq:sz-hat-from-sz1} in the observational environment using \eqref{eq:t-u-image-equality}. Proof for \eqref{eq:sz-hat-from-sz2} in interventional environments follows similar arguments. In order to show \eqref{eq:sz-hat-from-sz1}, i.e.,
\begin{align}
    s_{\hat Z}(\hat z) = H(U) \cdot s_{Z}(z)
\end{align}
where $z, x, \hat{z}$ are the samples of $Z, X, \hat{Z}(U, X)$, which are related through $x = Tz = U\hat{z}$. We start by investigating the term $H(U) \cdot s_{Z}(z)$.
\begin{align}
    H(U) \cdot s_{Z}(z)
    \overset{\eqref{eq:H-definition}}{=} (T^{+} U)^{\top} \cdot s_{Z}(z)
    \overset{\eqref{eq:sz-from-sx1}}{=} (T^{+} U)^{\top} \cdot T^{\top} \cdot s_{X}(x)
    = (T T^{+} U)^{\top} \cdot s_{X}(x) \ . \label{eq:H-def-proof-tmp1}
\end{align}
\begin{proposition}\label{prop:pseudo-inverse-left-identity}
    If $A, B \in \R^{d \times n}$ satisfy $\image(A) = \image(B)$, then we have
    \begin{align}
        A A^{+} B = B \ .
    \end{align}
\end{proposition}
\begin{proof}
    If $\image(A) = \image(B)$, then the column spaces of $A$ and $B$ are equal. This implies that column vectors of $B$ can be written as linear combinations of the column vectors of $A$. Formally, there exists $F \in \R^{n \times n}$ such that we have
    \begin{align}
        B = A F \ . \label{eq:prop-pili-eq1}
    \end{align}
    Next, from the properties of Moore-Penrose inverse, we have
    \begin{align}
        A A^+ A = A \ . \label{eq:prop-pili-eq2}
    \end{align}
    By post-multiplying both sides with $F$, we get
    \begin{align}
        A A^{+} B
        \overset{\eqref{eq:prop-pili-eq1}}{=} A A^{+} A F
        \overset{\eqref{eq:prop-pili-eq2}}{=} A F
        \overset{\eqref{eq:prop-pili-eq1}}{=} B \ . \label{eq:prop-pili-eq3}
    \end{align}
\end{proof}
Since we found $\image(T) = \image(U)$ in \eqref{eq:t-u-image-equality}, using Proposition~\ref{prop:pseudo-inverse-left-identity}, we get
\begin{align}
    H(U) s_{Z}(z) \overset{\eqref{eq:H-def-proof-tmp1}}{=} (T T^{+} U)^{\top} s_{X}(x) = U^{\top}s_{X}(x) \ .
\end{align}
By \eqref{eq:sz-hat-from-sx1}, $s_{\hat Z}(\hat z)=U^{\top}s_{X}(x)$, which completes the proof for \eqref{eq:sz-hat-from-sz1}.

Finally, we prove the second part, $H(U)$ is invertible, by contradiction. Suppose that $[H(U)]^{\top} = T^{+} U$ is not invertible for some $U \in \mcU$. This is true if and only if $T^{+} U$ has a non-trivial kernel, that is,
\begin{align}
    \ker(T^{+} U) \neq \{0\} \ , \label{eq:H-proof-kernel-tmp1}
\end{align}
where
\begin{align}
    \ker(T^{+} U) = \{ z \in \R^{n} : \; T^{+} U z = 0 \} \ .
\end{align}
By excluding $0$, the trivial member of both image and kernel, \eqref{eq:H-proof-kernel-tmp1} implies
\begin{align}
    \exists z \in \R^{n} \setminus \{0\} &: \; T^{+} U z = 0 \ . \label{eq:H-proof-kernel-tmp3}
\end{align}
Since $U \in \R^{d \times n}$ with $\rank(U) = n$, $U$ is an injective linear map with domain $\R^n$. Thus, for all $z \in \R^n \setminus \{0\}$, there exists $x \in \image(U) \setminus \{0\}$ such that $x = Uz$. Hence, \eqref{eq:H-proof-kernel-tmp3} is equivalent to
\begin{align}
    \exists x \in \image(U) \setminus \{0\} &: \; T^{+} x = 0 \ .
\end{align}
The kernel of $T^{+}$ consists of all $x \in \R^d$ such that $T^{+} x = 0$. Then, we get
\begin{align}
    \exists x \in \image(U) \setminus \{0\} &: \; x \in \ker(T^{+}) \ ,
\end{align}
which implies
\begin{align}
    \image(U) \cap \ker(T^{+}) &\neq \{0\} \ . \label{eq:H-proof-kernel-tmp2}
\end{align}
The kernel of Moore-Penrose inverse of a matrix equals the kernel of its transpose, that is, $\ker(T^{+}) = \ker(T^{\top})$. Hence, \eqref{eq:H-proof-kernel-tmp2} is equivalent to
\begin{align}
    \image(U) \cap \ker(T^{\top}) &\neq \{0\} \ .
\end{align}
However, note that for any matrix $A \in \R^{d \times n}$ we have $\ker(A^{\top}) \cap \image(A) = \{0\}$. From \eqref{eq:t-u-image-equality}, we know that $\image(T) = \image(U)$, which implies $\image(U) \cap \ker(T^{\top}) = \{0\}$. This is a contradiction, therefore $[H(U)]^{\top}$ is invertible for all $U \in \mcU$. A matrix is invertible if and only if its transpose is invertible, and this concludes the proof.



\subsection{Proof of Lemma~\ref{lemma:pq-parent-dependence}}\label{proof:pq-parent-dependence}

For $i \in [n]$ define
\begin{align}
     h(z_i, z_{\Pa(i)}) \triangleq \frac{q_Z(z_i\mid z_{\Pa(i)})}{p_Z(z_i\mid z_{\Pa(i)})} \ . \label{eq:h-pq-parent-dependence}
\end{align}
We start by showing that $h(z_i, z_{\Pa(i)})$ varies with $z_i$. We prove it by contradiction. Assume the contrary, i.e., let $h(z_i, z_{\Pa(i)}) = h(z_{\Pa(i)})$. By rearranging \eqref{eq:h-pq-parent-dependence} we have
\begin{align}\label{eq:misc1}
    q_Z(z_i \mid z_{\Pa(i)}) = h(z_{\Pa(i)}) p_Z(z_i \mid z_{\Pa(i)}) \ .
\end{align}
Fix a realization of $z_{\Pa(i)} = z_{\Pa(i)}^*$, and integrate both sides of \eqref{eq:misc1} with respect to $z_i$. Since both $p_Z$ and $q_Z$ are pdfs, we have
\begin{align}
    1 &= \int_{\R}  q_Z(z_i \mid z_{\Pa(i)}^*)\mathrm{d}{z_i}
    = \int_{\R}  h(z_{\Pa(i)}^*) p_Z(z_i \mid z_{\Pa(i)}^*) \mathrm{d}{z_i} \\
    &= h(z_{\Pa(i)}^*) \int_{\R} p_Z(z_i \mid z_{\Pa(i)}^*) \mathrm{d}{z_i} \\
    &= h(z_{\Pa(i)}^*)\ .
\end{align}
This identity implies that $p_Z(z_i \mid z_{\Pa(i)}^*) = q_Z(z_i \mid z_{\Pa(i)}^*)$ for any arbitrary realization $z_{\Pa(i)}^*$. This contradicts with the premise that observational and interventional distributions are distinct. As a result, in order to check if a model satisfies Assumption~\ref{assumption:pq-parent-dependence}, it suffices to investigate if for a given $i\in[n]$, $h$ is not invariant with respect to $z_k$ for $k\in\Pa(i)$. To this end, from \eqref{eq:h-pq-parent-dependence} we know that $h(z_i, z_{\Pa(i)})$ varies with $z_k$ if and only if
\begin{align}\label{eq:pq-parent-dependence-suff}
    \frac{\partial}{\partial z_i} \log h(z_i, z_{\Pa(i)})
    \overset{\eqref{eq:h-pq-parent-dependence}}{=}
    \frac{\frac{\partial q_Z(z_i \mid z_{\Pa(i)})}{\partial z_i}}{q_Z(z_i \mid z_{\Pa(i)})} -
    \frac{\frac{\partial p_Z(z_i \mid z_{\Pa(i)})}{\partial z_i}}{p_Z(z_i \mid z_{\Pa(i)})}
    \neq 0\ .
\end{align}
Next, we investigate the sufficient conditions listed in the Lemma~\ref{lemma:pq-parent-dependence}.

\subsubsection{Hard Interventions}\label{proof:pq-parent-dependence-hard}
Under hard interventions, note that for any $k\in\Pa(i)$,
\begin{align}
    \frac{\partial}{\partial z_k} p_Z(z_i \mid z_{\Pa(i)}) \neq 0 \ , 
    \quad \mbox{and} \quad
    \frac{\partial}{\partial z_k} q_Z(z_i) = 0\ .
\end{align}
Then, it follows directly from \eqref{eq:pq-parent-dependence-suff} that
\begin{align}
    \frac{\partial}{\partial z_k} \log h(z_i, z_{\Pa(i)})
    &= \underset{= 0}{\underbrace{\frac{\frac{\partial q_Z}{\partial z_k}(z_i)}{q_Z(z_i)}}}
    - \underset{\neq 0}{\underbrace{\frac{\frac{\partial p_Z}{\partial z_k}(z_i \mid z_{\Pa(i)})}{p_Z(z_i \mid z_{\Pa(i)})}}}
    \neq 0\ .
\end{align}
Thus, any model with hard interventions satisfies Assumption~\ref{assumption:pq-parent-dependence}.

\subsubsection{Additive Noise Models}\label{proof:pq-parent-dependence-additive}
Consider a node $k \in [n]$, for which the additive noise model is given by
\begin{align}
    z_i &= f_p(z_{\Pa(i)}) + N_p\  , \label{eq:additive-model-obs}\\
    z_i &= f_q(z_{\Pa(i)}) + N_q\  , \label{eq:additive-model-int}
\end{align}
in which we have used shorthands $\{f_p, f_q, N_p, N_q\}$ for $\{f_{p,k}, f_{q,k}, N_{p,k}, N_{q,k}\}$ that were defined in \eqref{eq:additive}. By defining $h_p$ and $h_q$ as the pdfs of $N_p$ and $N_q$, respectively, \eqref{eq:additive-model-obs} and \eqref{eq:additive-model-int} imply that
\begin{align}
    p_Z(z_i \mid z_{\Pa(i)}) = h_p(z_i - f_p(z_{\Pa(i)})\ , \quad \mbox{and} \quad 
    q_Z(z_i \mid z_{\Pa(i)}) = h_q(z_i - f_q(z_{\Pa(i)})\ . \label{eq:additive-model-pq}
\end{align}
Denote the score functions associated with $h_p$ and $h_q$
\begin{align}\label{eq:def-r-additive}
    r_p(u) \triangleq \frac{{\rm d}}{{\rm d} u} \log h_p(u) = \frac{h_p'(u)}{h_p(u)}\ , \quad \mbox{and} \quad
    r_q(u) \triangleq \frac{{\rm d}}{{\rm d} u} \log h_q(u) = \frac{h_q'(u)}{h_q(u)}\ .
\end{align}
We will prove \eqref{eq:pq-parent-dependence-suff} holds by contradiction. Assume that the contrary and let
\begin{align}\label{eq:pq-parent-dependence-additive-contrary}
    \frac{\frac{\partial q_Z(z_i \mid z_{\Pa(i)})}{\partial z_k}}{q_Z(z_i \mid z_{\Pa(i)})}
    =
    \frac{\frac{\partial p_Z(z_i \mid z_{\Pa(i)})}{\partial z_k}}{p_Z(z_i \mid z_{\Pa(i)})}\ .
\end{align}
From \eqref{eq:additive-model-pq} and \eqref{eq:def-r-additive}, for the numerators in \eqref{eq:pq-parent-dependence-additive-contrary} we have,
\begin{align}
    \frac{\partial p_Z(z_i \mid z_{\Pa(i)})}{\partial z_k}
    &= - \frac{\partial f_p(z_{\Pa(i)})}{\partial z_k} h_p'(z_i - f_p(z_{\Pa(i)})) \ , \\
    \mbox{and} \quad   \frac{\partial q_Z(z_i \mid z_{\Pa(i)})}{\partial z_k}
    &= - \frac{\partial f_q(z_{\Pa(i)})}{\partial z_k} h_q'(z_i - f_q(z_{\Pa(i)})) \ .
\end{align}
Hence, the identity in \eqref{eq:pq-parent-dependence-additive-contrary} can be written as
\begin{align}
    \frac{\partial f_p(z_{\Pa(i)})}{\partial z_k} r_p(z_i - f_p(z_{\Pa(i)})) &=
    \frac{\partial f_q(z_{\Pa(i)})}{\partial z_k} r_q(z_i - f_q(z_{\Pa(i)})) \ . \label{eq:pre-rearrange-additive-proof}
\end{align}
Define $n_p$ and $n_q$ as the realizations of $N_p$ and $N_q$ when $z_i=z_i$ and $z_{\Pa(i)}=z_{\Pa(i)}$. By defining $\delta(z_{\Pa(i)}) \triangleq f_p(z_{\Pa(i)}) - f_q(z_{\Pa(i)})$, we have
\begin{align}
    n_q = n_p + \delta(z_{\Pa(i)}) \ .
\end{align}
Hence, \eqref{eq:pre-rearrange-additive-proof} becomes equivalent to
\begin{align}\label{eq:pre-rearrange-additive-proof2}
    \frac{\partial f_p(z_{\Pa(i)})}{\partial z_k} r_p(n_p) &=
    \frac{\partial f_q(z_{\Pa(i)})}{\partial z_k} r_q(n_p + \delta(z_{\Pa(i)})) \ .     
\end{align}
Note that $\frac{\partial f_p(z_{\Pa(i)})}{\partial zki}$ is a nonzero continuous function. Hence, there exists an interval $\Phi\subseteq\R^{|\Pa(i)|}$ over which $\frac{\partial f_p(z_{\Pa(i)})}{\partial z_k} \neq 0$. Likewise, $r_p(n_p)$ cannot be constantly zero over all possible intervals $\mathcal{E} \subseteq \R$. This is because otherwise, it would have to necessarily be a constant zero function (since it is analytic), which is an invalid score function. Hence, there exists an open interval $\mathcal{E} \subseteq \mathbb{R}$ over which $r_p(n_p)$ is non-zero for all $n_p \in \mcE$. Then, we can rearrange \eqref{eq:pre-rearrange-additive-proof2}
\begin{align}\label{eq:rp-rq-equals-d-fq-d-fp}
    \frac{r_p(n_p)}{r_q(n_p+\delta(z_{\Pa(i)}))} &=
    \frac{\frac{\partial f_q(z_{\Pa(i)})}{\partial z_k}}{\frac{\partial f_p(z_{\Pa(i)})}{\partial z_k}} \ , \quad \forall (n_p,z_{\Pa(i)})\in \mcE \times \Phi \ .
\end{align}
Note that right-hand side (RHS) of \eqref{eq:rp-rq-equals-d-fq-d-fp} is not a function of $n_p$. Then, taking the derivative of both sides with respect to $n_p$, we get
\begin{align}\label{eq:d-log-rp-d-log-rq-with-delta}
    \frac{r_p'(n_p)}{r_p(n_p)} &=
    \frac{r_q'(n_p + \delta(z_{\Pa(i)}))}{r_q(n_p + \delta(z_{\Pa(i)}))} \ .
\end{align}
In the next step, we show that $\delta$ is not a constant function. We prove this by contradiction. Suppose that $\delta(z_{\Pa(i)}) = \delta^*$ is a constant function. Then, the gradients of $f_p$ and $f_q$ are equal. From \eqref{eq:rp-rq-equals-d-fq-d-fp}, this implies that
\begin{align}
    r_p(n_p) &= r_q(n_p + \delta^*) \ , \quad \forall n_p \in \mcE \ .
\end{align}
Since $r_p(n_p)$ and $r_q(n_p + \delta^*)$ are analytic functions that agree on an open interval of $\R$, they are equal for all $n_p \in \R$. This implies that $h_p(n_p) = d^* h_q(n_p+\delta^*)$ for some $d^*$ constant. Since $h_p$ and $h_q$ are pdfs, $d^*=1$ is the only choice that maintains $h_p$ and $h_q$ are pdfs. Therefore, $h_p(n_p)=h_q(n_p+\delta^*)$. However, using \eqref{eq:additive-model-pq}, this implies that $p_Z(z_i \mid z_{\Pa(i)}) = q_Z(z_i \mid z_{\Pa(i)})$, which contradicts the premise that an intervention changes the causal mechanism of target node $k$. Therefore $\delta$ is a continuous, non-constant function, and its image over $z_{\Pa(i)}\in\Phi$ includes an open interval $\Theta\subseteq\R$. With this result in mind, we return to \eqref{eq:d-log-rp-d-log-rq-with-delta}. Consider a fixed realization $n_p = n_p^*$ and denote the value of left-hand side (LHS) for $n_p^*$ by $C$. By defining $u \triangleq \delta(z_{\Pa(i)})$, we get
\begin{align}
    C &= \frac{r_q'(n_p^* + u)}{r_q(n_p^* + u)} \ , \qquad \forall u\in\Theta \ .
\end{align}
This is only possible if $r_q$ is an exponential function, i.e., $r_q(u) = k_1 \exp(Cu)$ over interval $u \in \Theta$. Since $r_q$ is an analytic function, it is, therefore, exponential over entire $\R$. Then, the associated pdf must have the form $h_q(u) = k_2 \exp((k_1/C) \exp(Cu))$. However, the integral of this function over the entire domain $\R$ diverges. Hence, it is not a valid pdf, rendering a contradiction. Hence, the additive noise model satisfies Assumption~\ref{assumption:pq-parent-dependence}.

\subsubsection{Multiplicative Noise Models}\label{proof:pq-parent-dependence-multiplicative}
Consider a node $i \in [n]$, for which the multiplicative noise model is given by
\begin{align}
    z_i &= f_p(z_{\Pa(i)}) \cdot N_p\ , \label{eq:multiplicative-model-obs} \\
    z_i &= f_q(z_{\Pa(i)}) \cdot N_q\ , \label{eq:multiplicative-model-int}
\end{align}
for observational and interventional causal mechanisms, respectively. The pdfs and score functions of the noise terms are defined similarly to those in Section~\ref{proof:pq-parent-dependence-additive}. For this model to have a well-defined pdf, we need to have $f_p$ and $f_q$ be non-zero everywhere. Then, \eqref{eq:multiplicative-model-obs} and \eqref{eq:multiplicative-model-int} imply that we have
\begin{align}
    p_Z(z_i \mid z_{\Pa(i)}) &= \frac{1}{|f_p(z_{\Pa(i)})|} h_p\left(\frac{z_i}{f_p(z_{\Pa(i)})}\right)\ , \\  
    q_Z(z_i \mid z_{\Pa(i)}) &= \frac{1}{|f_q(z_{\Pa(i)})|} h_q\left(\frac{z_i}{f_q(z_{\Pa(i)})}\right)\ .
\end{align}
We prove that \eqref{eq:pq-parent-dependence-suff} holds by contradiction. Suppose the contrary and let
\begin{align}
    0 &= \frac{\partial}{\partial z_k} \log h(z_i, z_{\Pa(i)})
    = \frac{\partial}{\partial z_k} \log q_Z(z_i \mid z_{\Pa(i)}) -
    \frac{\partial}{\partial z_k} \log p_Z(z_i \mid z_{\Pa(i)})\ ,
\end{align}
or equivalently,
\begin{align}\label{eq:pq-parent-dependence-additive-contrary-mult}
    \frac{\partial}{\partial z_k} \log p_Z(z_i \mid z_{\Pa(i)}) &=
    \frac{\partial}{\partial z_k} \log p_Z(z_i \mid z_{\Pa(i)})\ .
\end{align}
Let us define $\delta(z_{\Pa(i)}) \triangleq \frac{f_p(z_{\Pa(i)})}{f_q(z_{\Pa(i)})}$. We will show that \eqref{eq:pq-parent-dependence-additive-contrary-mult} cannot be true in two steps. 
\paragraph{Step 1:} Suppose that $\delta(z_{\Pa(i)})=\delta^*$ is a constant function. In this case, by definition, we have
\begin{align}\label{eq:multiplicative-constant-delta}
    \frac{h_p\left(\frac{z_i}{f_p(z_{\Pa(i)})}\right)}{h_q\left(\delta^* \frac{z_i}{f_q(z_{\Pa(i)})}\right)} = \delta^* h(z_i) \ .
\end{align}
Now, suppose that $f_p(z_{\Pa(i)})$ is continuous on an interval $[c_1,c_2]$, where $0 < c_1 < c_2$. Let $z_{\Pa(i)}=\varphi_1$ satisfies $f_p(\varphi_1)=c_1$. Consider an arbitrary $c_0 \in [c_1,c_2]$ and let $f_p(\varphi_0)=c_0$. Then, computing \eqref{eq:multiplicative-constant-delta} for $(z_i,z_{\Pa(i)})=(z_i,\varphi_1)$ and $(\frac{c_0}{c_1}z_i),\varphi_0$ values, we obtain
\begin{align}
    \frac{h_p(\frac{z_i}{c_1})}{h_q(\delta^* \frac{z_i}{c_1})} &= \delta^* h(z_i) \ , \\
    \delta^* h\left(\frac{c_0}{c_1}z_i\right) \frac{h_p(\frac{c_0 z_i}{c_1 c_0})}{h_q(\delta^* \frac{c_0 z_i}{c_1 c_0})} &= \frac{h_p(\frac{z_i}{c_1})}{h_q(\delta^* \frac{z_i}{c_1})} = \delta^* h(z_i) \ .
\end{align}
This implies that $h(z_i)=h(\frac{c_0}{c_1}z_i)$ for all values of $\frac{c_0}{c_1} \in [1,\frac{c_2}{c_1}]$. Then, $h$ is a constant function, i.e., the ratio of $p_Z(z_i \mid z_{\Pa(i)})$ and $q_Z(z_i \mid z_{\Pa(i)})$ is constant, which has to be $1$, and contradicts the premise that they are distinct. Hence, we conclude that $\delta(z_{\Pa(i)})$ is not a constant function, which we investigate next. 
\paragraph{Step 2:} Suppose that $\delta(z_{\Pa(i)})$ is not a constant function. We can compute LHS of \eqref{eq:pq-parent-dependence-additive-contrary-mult} as
\begin{align}
    \frac{\partial}{\partial z_k} \log p_Z(z_i \mid z_{\Pa(i)})
    &= - \frac{\frac{\partial f_p(z_{\Pa(i)})}{\partial z_k}}{f_p(z_{\Pa(i)})} - r_p\left(\frac{z_i}{f_p(z_{\Pa(i)})}\right) \cdot z_i \cdot \frac{\frac{\partial f_p(z_{\Pa(i)})}{\partial z_k}}{\left(f_p(z_{\Pa(i)})\right)^2} \\
    &= - \frac{\frac{\partial f_p(z_{\Pa(i)})}{\partial z_k}}{f_p(z_{\Pa(i)})} \left( 1 + r_p\left(\frac{z_i}{f_p(z_{\Pa(i)})}\right) \cdot \frac{z_i}{f_p(z_{\Pa(i)})} \right) \\
    &= - \frac{\frac{\partial f_p(z_{\Pa(i)})}{\partial z_k}}{f_p(z_{\Pa(i)})} \left( 1 + r_p(n_p) \cdot n_p \right)  \ . \label{eq:multiplicative-pre-ratio-p}
\end{align}
Similarly, for RHS of \eqref{eq:pq-parent-dependence-additive-contrary-mult} we have
\begin{align}
     \frac{\partial}{\partial z_k} \log p_q(z_i \mid z_{\Pa(i)}) & =- \frac{\frac{\partial f_q(z_{\Pa(i)})}{\partial z_k}}{f_q(z_{\Pa(i)})} \left( 1 + r_q(n_q) \cdot n_q \right)  \\
     & =- \frac{\frac{\partial f_q(z_{\Pa(i)})}{\partial z_k}}{f_q(z_{\Pa(i)})} \left( 1 + r_q(n_p \delta(z_{\Pa(i)})) \cdot n_p\delta(z_{\Pa(i)}) \right)  \ . \label{eq:multiplicative-pre-ratio-q}
\end{align}
Note that $f_p$ is non-zero everywhere, $f_q$ is continuously differentiable, and $(1+r_q(n_q)\cdot n_q)$ is an analytic function. Then, by following a similar line of arguments as in the proof of the additive noise model in Section~\ref{proof:pq-parent-dependence-additive}, there exist open intervals $\mcE$ and $\Phi$ such that we can substitute \eqref{eq:multiplicative-pre-ratio-p} and \eqref{eq:multiplicative-pre-ratio-q} to \eqref{eq:pq-parent-dependence-additive-contrary-mult} and rearrange it as
\begin{align}\label{eq:fp-fq-rq-rp-for-multiplicative}
   \frac{\frac{\partial f_p(z_{\Pa(i)})}{\partial z_k}}{\frac{\partial f_q(z_{\Pa(i)})}{\partial z_k}} \cdot \frac{f_q(z_{\Pa(i)})}{f_p(z_{\Pa(i)})} 
    &= \frac{1 + r_q(n_p \delta(z_{\Pa(i)})) \cdot n_p\delta(z_{\Pa(i)})}{ 1 + r_p(n_p) \cdot n_p } \ .
\end{align}
LHS of \eqref{eq:fp-fq-rq-rp-for-multiplicative} is not a function of $n_p$. Then, taking the derivative of both sides with respect to $n_p$, and rearranging the derivative of RHS of \eqref{eq:fp-fq-rq-rp-for-multiplicative}, we get
\begin{align}
    \frac{r_p(n_p) + n_p r_p'(n_p)}{1 + n_p r_p(n_p)} &=
    \frac{\delta(z_{\Pa(i)}) r_q(\delta(z_{\Pa(i)}) n_p) + \delta(z_{\Pa(i)})^2 n_p r_q'(\delta(z_{\Pa(i)}) n_p)}{1 + \delta(z_{\Pa(i)}) n_p r_q(\delta(z_{\Pa(i)}) n_p)} \ .
\end{align}
Since $\delta$ is not a constant function, the image of $\delta$ over $z_{\Pa(i)}\in\Phi$ contains an open interval $\Theta\subseteq\R$. Consider a fixed realization $n_p = n_p^*$ and denote the value of LHS for $n_p^*$ by $C$. By defining $u \triangleq \delta(z_{\Pa(i)})$, we have
\begin{align}\label{eq:C-multiplicative}
    C &=
    \frac{u n_p^* r_q(u n_p^*) + u^2 (n_p^*)^2 r_q'(u n_p^*)}{1 + u n_p^* r_q(u n_p^*)} =
    n_q \frac{r_q(n_q) + n_q r_q'(n_q)}{1 + n_q r_q(n_q)} \ .
\end{align}
Define $t(n_q) \triangleq 1 + n_q r_q(n_q)$. Then, from the identity in \eqref{eq:C-multiplicative} we have
\begin{align}
    C &= n_q \frac{t'(n_q)}{t(n_q)}
    \implies t(n_q) = k (n_p)^C 
    \implies r_q(n_q) = k (n_p)^{C-1} - \frac{1}{n_p} \ .
\end{align}
Hence,
\begin{align}
    \frac{1}{|n_q|} \exp\left(\frac{k}{C} (n_p)^C\right) \ .
\end{align}
The integral of $h_q$ over the entire domain $\R$ diverges. Thus, it cannot be a valid pdf, which is a contradiction. Hence, the multiplicative noise model satisfies Assumption~\ref{assumption:pq-parent-dependence}.


\subsection{Proof of Theorem~\ref{theorem:condition-assumption}}\label{proof:condition-assumption}

We show that for each node $i \in [n]$, the condition
\begin{align}
    \left(c^{\top}\cdot s_{Z} \right) \overset{p_Z}{\neq} \left(c^{\top}\cdot s^{m_i}_{Z} \right) \ , \;\; \forall c \in \mcC_i\ , 
    \label{eq:assumption_sc:node_level}
\end{align}
holds if and only if the following continuum of equations admit their solutions $c$ in $\R^n\backslash \mcC_i$.
\begin{align}
  \left\{
   \begin{array}{l}
         c_i - c^{\top} \cdot \nabla_z f_{p,i}(\varphi)  = 0  \\
         \\ 
         c_i - c^{\top} \cdot \nabla_z f_{q,i}(\varphi)  = 0
   \end{array}
   \right.\ , \qquad \qquad \forall \varphi \in\R^{|\Pa(i)|}\ , \label{eq:assumption_sc:condition_node_level}
\end{align}
in which shorthand $\varphi$ is used for $z_{\Pa(i)}$. We first note that the condition in \eqref{eq:assumption_sc:node_level} can be equivalently stated as follows. For all $c \in \mcC_i$ we have
\begin{align}\label{eq:equivalent_statement2}
 \left(c^{\top}\cdot \nabla_z \log p_{Z}(z_i \mid z_{\Pa(i)})\right) \overset{p_Z}{\neq}  \left(c^{\top}\cdot \nabla_z \log q_{Z}(z_i \mid z_{\Pa(i)})\right)   \ .
\end{align}
This can be readily verified by noting that 
\begin{align}
    c^{\top}\Big[s_{Z}(z) -  s^{m_i}_{Z}(z) \Big] \overset{\eqref{eq:s_z_decompose_obs},\eqref{eq:s_z_decompose_int}}&{=}  c^{\top}\Big[\nabla_z \log p_{Z}(z_i \mid z_{\Pa(i)}) - \nabla_z \log q_{Z}(z_i \mid z_{\Pa(i)})\Big]   \ . \label{eq:equivalent_statement1}
\end{align}
The additive noise model for node $i$ is given by 
\begin{align}
    Z_i &= f_p(Z_{\Pa(i)}) + N_p\  ,  \label{eq:additive-model-obs-2} \\
    Z_i &= f_q(Z_{\Pa(i)}) + N_q\  ,  \label{eq:additive-model-int-2}
\end{align}
in which we have used shorthands $\{f_p, f_q, N_p, N_q\}$ for $\{f_{p,i}, f_{q,i}, N_{p,i}, N_{q,i}\}$ that were defined in \eqref{eq:additive}. By defining $h_p$ and $h_q$ as the pdfs of $N_p$ and $N_q$, respectively, \eqref{eq:additive-model-obs-2} and \eqref{eq:additive-model-int-2} imply that
\begin{align}\label{eq:p-h-connection}
    p_Z(z_i \mid z_{\Pa(i)}) = h_p(z_i - f_p(z_{\Pa(i)})\ , \quad \mbox{and} \quad 
    q_Z(z_i \mid z_{\Pa(i)}) = h_q(z_i - f_q(z_{\Pa(i)})\ .
\end{align}
Denote the score functions associated with $h_p$ and $h_q$
\begin{align}\label{eq:def-r}
    r_p(u) \triangleq \frac{{\rm d}}{{\rm d} u} \log h_p(u) = \frac{h_p'(u)}{h_p(u)}\ , \quad \mbox{and} \quad
    r_q(u) \triangleq \frac{{\rm d}}{{\rm d} u} \log h_q(u) = \frac{h_q'(u)}{h_q(u)}\ .
\end{align}
Define $n_p$ and $n_q$ as the realizations of $N_p$ and $N_q$ when $Z_k=z_k$ and $Z_{\Pa(k)}=z_{\Pa(k)}$. By defining $\delta(z_{\Pa(k)}) \triangleq f_p(z_{\Pa(k)}) - f_q(z_{\Pa(k)})$, we have
\begin{align}
    n_q = n_p + \delta(z_{\Pa(k)}) \ .
\end{align}
Using \eqref{eq:def-r} and \eqref{eq:p-h-connection}, we can express the relevant entries of $\nabla_{z} \log p_{Z}(z_{i} \mid z_{\Pa(i)})$ and $\nabla_{z} \log q_{Z}(z_{i} \mid z_{\Pa(i)})$ as
\begin{align}
    [\nabla_{z} \log p_{Z}(z_{i} \mid z_{\Pa(i)})]_j
    &= \begin{cases} 
        r_p(n_p)  \ , &j = i \ , \\
        -\pdv{f_p(z_{\Pa(i)})}{z_j} r_p(n_p)  \ , &j \in \Pa(i) \ , \\
        0 \ , &j \notin  \overline{\Pa}(i)  \ ,
    \end{cases} \label{eq:a2_score_obs}  \\
    \mbox{and} \quad 
    [\nabla_{z} \log q_{Z}(z_{i} \mid z_{\Pa(i)})]_j
    &= \begin{cases} 
        r_q(N_q)  \ , &j = i \ , \\
        -\pdv{f_q(z_{\Pa(i)})}{z_j}  r_q(N_q)  \ , &j \in \Pa(i) \  , \\
        0 \ , &j \notin  \overline{\Pa}(i) \ .
    \end{cases} \label{eq:a2_score_int}
\end{align}
By substituting \eqref{eq:a2_score_obs}--\eqref{eq:a2_score_int} in \eqref{eq:equivalent_statement2} and rearranging the terms, the statement in~\eqref{eq:equivalent_statement2} becomes equivalent to following statement. For all $c \in \mcC_i$, there exist $n_p \in \R$ and $z_{\Pa(i)}\in\R^{|\Pa(i)|}$ such that 
\begin{align} \label{eq:equivalent_statement3} 
    r_p(n_p) \left(c_i- \sum_{j \in {\Pa}(i)} c_j\cdot \pdv{f_p(z_{\Pa(i)})}{z_j}\right) \neq r_q(n_p+\delta(z_{\Pa(i)})) \left(c_i- \sum_{j \in {\Pa}(i)} c_j\cdot \pdv{f_q(z_{\Pa(i)})}{z_j}\right)\ ,   
\end{align}
which by using the shorthand $\varphi$ for $z_{\Pa(i)}$ can be compactly presented as 
\begin{align}\label{eq:equivalent_statement4}
    r_p(n_p) \cdot \left[c_i-c^{\top} \cdot \nabla_z f_p(\varphi)\right] \neq r_q(n_p+\delta(\varphi)) \cdot \left[c_i-c^{\top} \cdot \nabla_z  f_q(\varphi)\right] \ , \quad \exists n_p \in \R, \;\; \exists \varphi\in\R^{|\Pa(i)|}\ , \;\; \forall c \in \mcC_i \ .
\end{align}
Hence, Assumption~\ref{assumption:change} is equivalent to the statement in~\eqref{eq:equivalent_statement4}, which we use for the rest of the proof.

\paragraph{Sufficient condition.} We show that if~\eqref{eq:assumption_sc:condition_node_level} admit solutions $c$ only in $\R^n\backslash \mcC_i$, then the statement in~\eqref{eq:equivalent_statement4} holds. By contradiction, assume that there exists $c^*\in\mcC_i$ for which
\begin{align}\label{eq:equivalent_statement5}
    r_p(n_p) \cdot \left[c^*_i-(c^*){\top} \cdot \nabla_z f_p(\varphi)\right] = r_q(n_p+\delta(\varphi))  \cdot \left[c^*_i-(c^*)^{\top} \cdot \nabla_z  f_q(\varphi)\right] \ , \quad \forall n_p \in \R  \ , \;\; \forall \varphi\in\R^{|\Pa(i)|}\ .
\end{align}
We show that $c^*\in\mcC_i$ is also a solution to~\eqref{eq:assumption_sc:condition_node_level}, contradicting the premise. In order to show that $c^*\in\mcC_i$ is also a solution to~\eqref{eq:assumption_sc:condition_node_level}, suppose, by contradiction, that \eqref{eq:equivalent_statement5} holds, and there exists $\varphi^*\in\R^{|\Pa(i)|}$ corresponding to which
\begin{align}\label{eq:contradiction_premise}
    (c^*)^{\top}\cdot \nabla_z f_q(\varphi^*)- c^*_i \neq 0\ .
\end{align}
Note that $f_p$ is a continuously differentiable function and also a function of $z_j$ for all $j \in \Pa(i)$. Hence, there exists an open set $\Phi \subseteq \mathbb{R}^{|\Pa(i)|}$ for $\varphi$ for which $(c^*)^{\top}\cdot \nabla_z f_p(\varphi^*)- c^*_i$ is non-zero everywhere in $\Phi$. Likewise, $r_p$ cannot constantly be zero over all possible intervals $\mathcal{E}$. This is because otherwise, it would have to necessarily be a constant zero function (since it is analytic), which is an invalid score function. Hence, there exists an open interval $\mathcal{E} \subseteq \mathbb{R}$ over which $r_p(n_p)$ is non-zero for all $\epsilon_i\in\mathcal{E}$. Subsequently, the left-hand side of \eqref{eq:equivalent_statement5} is non-zero over the Cartesian product $(n_p, \varphi) \in \mathcal{E} \times \Phi$. This means that if \eqref{eq:equivalent_statement5} is true, then both functions on its right-hand side must be also non-zero over $\mathcal{E} \times \Phi$. Hence, by rearranging the terms in~\eqref{eq:equivalent_statement5} we have
\begin{align}
   \frac{r_q(n_p+\delta(\varphi))}{r_p(n_p)} = \frac{c^*_i - (c^*)^{\top}\cdot \nabla_z  f_p(\varphi)}{ c^*_i - (c^*)^{\top}\cdot \nabla_z f_q(\varphi)}\ , \qquad \forall (n_p, \varphi) \in \mcE \times \Phi \ . \label{eq:ratio}
\end{align}
We show that there are no valid pdfs $h_p$ and $h_q$ for which \eqref{eq:ratio} is valid. We show this in two steps.

\noindent \textbf{Step (a).} First, we show that function $\delta$ cannot be a constant over $\Phi$ (interval specified above). Suppose the contrary and assume that $\delta(\varphi)= \delta^*$ for all $\varphi \in \Phi$. Hence, the gradient of $\delta$ is zero. Using the definition of $\delta$, this implies that
\begin{align}
    \nabla_z f_p(\varphi) = \nabla_z f_q(\varphi) \ , \quad \forall \varphi \in \Phi \ .
\end{align}
Then, by leveraging \eqref{eq:ratio}, we conclude that $r_p(n_p) = r_q(n_p + \delta^*)$ for all $n_p \in \mathcal{E}$. Since $r_p(n_p)$ and $r_q(n_p + \delta^*)$ are analytic functions that agree on an open interval of $\mathbb{R}$, they are equal for all $n_p \in \mathbb{R}$ as well. This implies that $h_p(n_p) = d^* h_q(n_p+\delta^*)$ for some $d^*$ constant. Since $h_p$ and $h_q$ are pdfs, the only choice is $d^*=1$ and $h_p(n_p)=h_q(n_p+\delta^*)$, then $p_Z(z_i \mid z_{\Pa(i)}) =  q_Z(z_i \mid z_{\Pa(i)})$, which contradicts the premise.

\noindent \textbf{Step (b).} Finally, we will show that \eqref{eq:ratio} cannot be true when $\delta$ is not a constant function over $\Phi$. Note that the right-hand side of \eqref{eq:ratio} is not a function of $n_p$. Then, taking the derivative of both sides with respect to $n_p$ and rearranging, we obtain
\begin{align}\label{eq:ratio-derivatives}
    \frac{r_p'(n_p)}{r_p(n_p)} &=
    \frac{r_q'(n_p + \delta(\varphi))}{r_q(n_p + \delta(\varphi))} \ , \quad \forall (n_p,\varphi)\in \mcE \times \Phi  \ .
\end{align}
Next, Consider a fixed realization $n_p=n_p^*$, and denote the value of LHS by $C$. Since $\delta$ is continuous and not constant over $\Phi$, its image contains an open interval $\Theta \subseteq \mathbb{R}$. Denoting $u\triangleq \delta(\varphi)$, we get
\begin{align}
    C = \frac{r_q'(n_p^* + u)}{r_q(n_p^* + u)} \ , \quad \forall u \in \Theta \ .
\end{align}
The only solution to this equality is that $r_q(n_p^* + u)$ is an exponential function, $r_q(u) = k_1\exp(Cu)$ over interval $u \in \Theta$. Since $r_q$ is an analytic function that equals to an exponential function over an interval, it is exponential over entire $\mathbb{R}$. This implies that the pdf $h_q(u)$ is of the form $h_q(u) = k_2 \exp((k_1/C) \exp(Cu))$. However, this cannot be a valid pdf since its integral over $\R$ diverges. Hence, \eqref{eq:ratio} is not true, and the premise that there exists $c^* \in \mcC_i$ and $\varphi\in\R^{|\Pa(i)|}$ corresponding to which \eqref{eq:contradiction_premise} holds is invalid, concluding that for all $\varphi\in\R^{|\Pa(i)|}$ we have $c^{\top}\cdot  \nabla_z  f_q(\varphi) = c_i$. Proving the counterpart identity $c^{\top}\cdot  \nabla_z  f_p(\varphi) = c_i$ follows similarly. Therefore, \eqref{eq:equivalent_statement5} implies
\begin{align}
  \left\{
   \begin{array}{l}
         c^*_i - (c^*)^{\top} \cdot \nabla_z f_p(\varphi)  = 0  \\
         \\ 
         c^*_i - (c^*)^{\top} \cdot \nabla_z f_q(\varphi)  = 0
   \end{array}
   \right.\ , \qquad \qquad \forall \varphi \in\R^{|\Pa(i)|}\ ,
\end{align}
which means that we have found a solution to ~\eqref{eq:assumption_sc:condition_node_level} that is not in $\R^n\backslash \mcC_i$, contradicting the premise. 

\paragraph{Necessary condition.} Assume that Assumption~\ref{assumption:change}, and equivalently, the statement in~\eqref{eq:equivalent_statement4} hold but \eqref{eq:assumption_sc:condition_node_level} has a solution $c^*$ in $\mcC_i$. Hence,
\begin{align}
     c^*_i - (c^*)^{\top} \cdot \nabla_z f_p(\varphi)  =  c^*_i - (c^*)^{\top} \cdot \nabla_z f_q(\varphi)= 0 \ , \quad \forall \varphi \in \R^{|\Pa(i)|} \ .
\end{align}
Then, multiplying left side by $r_p(n_p)$ and right side by $r_q(n_q)$, we obtain
\begin{align}
      r_p(n_p) \cdot \left[c^*_i-(c^*){\top} \cdot \nabla_z f_p(\varphi)\right] = r_q(n_p+\delta(\varphi))  \cdot \left[c^*_i-(c^*)^{\top} \cdot \nabla_z  f_q(\varphi)\right] =  0 \ , \quad \forall n_p \in \R \ , \;\; \forall \varphi\in\R^{|\Pa(i)|}\ ,
\end{align}
which implies
\begin{align}
   r_p(n_p) \cdot \left[c_i-c^{\top} \cdot \nabla_z f_p(\varphi)\right] = r_q(n_p+\delta(\varphi)  \cdot \left[c_i-c^{\top} \cdot \nabla_z f_q(\varphi)\right] \ , \quad \forall n_p \in \R \ , \;\; \forall \varphi\in\R^{|\Pa(i)|}\ , \;\; \forall c \in \mcC_i \ .
\end{align}
This contradicts Assumption~\ref{assumption:change} and completes the proof.

\subsection{Proof of Lemma~\ref{prop:nn}}\label{proof:nn}
\paragraph{Approach.}
We will use the same argument as at the beginning of the proof of Theorem~\ref{theorem:condition-assumption}. Specifically, we will show that for any node $i$ and two-layer NNs $f_{p,i}$ and $f_{q,i}$ with weight matrices $W^{p,i} $ and $W^{q,i}$, the following continuum of equations admit their solutions $c$ in $\R^n\backslash \mcC_i$.
\begin{align}
    \left\{
   \begin{array}{l}
         c_i - c^{\top} \cdot \nabla_z f_{p,i}(\varphi)  = 0  \\
         \\ 
         c_i - c^{\top} \cdot \nabla_z f_{q,i}(\varphi)  = 0
   \end{array}
   \right.\ , \qquad \qquad \forall \varphi \in\R^{|\Pa(i)|}\ ,
\end{align}
in which shorthand $\varphi$ is used for $z_{\Pa(i)}$. Hence, by invoking Theorem~\ref{theorem:condition-assumption}, Assumption~\ref{assumption:change} holds. 

\paragraph{Definitions.} Define $d_i \triangleq |\Pa(i)|$. Since $\max\{\rank(W^{p,i}), \rank(W^{q,i})\}=d_i$, without loss of generality, suppose that $W^{p,i}\in \R^{w_{p,i} \times d_i}$ has rank $d_i$. The rest of the proof follows similarly for the case $\rank( W^{q,i})=d_i$. We use shorthand $\{f, W, w, \nu, \nu_0\}$ for $\{f_{p,i}, W^{p,i}, w_{p,i}, \nu^{p,i}, \nu^{p,i}_0\}$ when it is obvious from context.

\paragraph{Parameterization.}
Note that $f$ can be represented by different parameterizations, some containing more hidden nodes than others. Without loss of generality, let $w$ be the fewest number of nodes that can represent $f$. This implies that the entries of $\nu$ are non-zero. Otherwise, if $\nu_i=0$, we can remove $i$-th hidden node and still have the same $f$. Similarly, the rows of $W$ are distinct. Otherwise, if there exist rows $W_i = W_j$ for distinct $i,j \in [w]$, removing $j$-th hidden node and using $(\nu_i + \nu_j)$ in place of $\nu_i$ results the same function as $f$ with $(w-1)$ hidden nodes. Similarly, we have $W_i \neq \boldsymbol{0}$ for all $i \in [w]$. Otherwise, we have $\nu_i \sigma(W_i \varphi) + \nu_0 = (\nu_0 + \frac{\nu_i}{2})$, and by removing $i$-th hidden node and using $(\nu_0 + \frac{\nu_i}{2})$ instead of $\nu_0$, we reach the same function as $f$ with $(p-1)$ hidden nodes. Finally, we have $W_i + W_j \neq \boldsymbol{0}$. Otherwise, we have $\nu_i \sigma(W_i \varphi) + \nu_j \sigma(W_j \varphi) + \nu_0 = (\nu_i - \nu_j) \sigma(W_i \varphi) + (\nu_0+\nu_j)$, and by removing $j$-th hidden node and using $(\nu_i-\nu_j)$ instead of $\nu_i$ and $(\nu_0+\nu_j)$ instead of $\nu_0$, we reach the same function as $f$ with $(w-1)$ hidden nodes. In summary, we have
\begin{align}
    W_i \neq \boldsymbol{0} \ , \;\; \nu_i \neq 0 \ ,  \quad  \forall i \in [w] \ , \quad \mbox{and}  \quad    W_i \pm W_j \neq \boldsymbol{0} \ , \quad \forall i, j \in [w] : i \neq j \ . \label{eq:W_construction}
\end{align}
We will show that there does not exist $c \in \mcC_i$ such that $c^{\top}\nabla_z f(\varphi) = c_i$ for all $\varphi \in \R^{d_i}$. 
Assume the contrary, and assume there exists $c^* \in \mcC_i$ such that
\begin{align}
    (c^*)^{\top} \nabla_z f(\varphi) = c^*_i \ , \quad \forall \varphi \in \R^{d_i} \ .
    \label{eq:nn-contrary}
\end{align}
This is equivalent to showing that there exists non-zero $b^* \in \R^{d_i}$ such that 
\begin{align}
    (b^*)^{\top} \nabla_{\varphi} f(\phi) = c_i^{*} \ , \quad \forall \varphi \in \R^{d_i} \ .
\end{align}
Based on \eqref{eq:nn-expression}, the gradient of $f(\varphi)$ is
\begin{align}
    \nabla_{\varphi} f(\varphi) = W^{\top} \cdot \diag(\nu) \cdot \sdot(W \varphi) \ . \label{eq:u-expression}
\end{align}
where $\diag(\nu)$ is the diagonal matrix with $\nu$ as its diagonal elements, and $\sdot$ is the derivative of the sigmoid function, applied element-wise to its argument. Hence, based on \eqref{eq:nn-contrary}, the contradiction premise is equivalent to having a non-zero $b^* \in \R^{d_i}$ such that
\begin{align}
    [\sdot(W \varphi)]^{\top} \cdot \diag(\nu) W b^* = c_i \ , \quad \forall \varphi \in \R^{d_i} \ .  \label{eq:nn-constant-requirement}
\end{align}
We note that since $W$ is full-rank and $\nu_{i}$ is non-zero for all $i \in [w]$, $(\diag(\nu) W)$ is full-rank as well and has a trivial null space. Subsequently, $b^* \in \R^{d_i}$ is non-zero if and only if $(\diag(\nu) W b^*) \in \R^{p}$ is non-zero. We will use the following lemma to show that \eqref{eq:nn-constant-requirement} cannot be true. This establishes that the contradiction premise is not correct, and completes the proof.
\begin{lemma}\label{lm:aux-nn}
    Let $u \in \R^p$ have non-zero entries with distinct absolute values, i.e., $u_i \neq 0$ and $|u_i| \neq |u_j|$ for all $i \neq j$, and $a \in \R$ be a constant. Then, for every non-zero vector $d \in \R^p$, there exists $\alpha \in \R$ such that $[\sdot(\alpha u)]^{\top} d \neq a$.
\end{lemma}
\begin{proof}
    See Appendix~\ref{sec:proof:lm:aux-nn}
\end{proof}
Let us define $\xi \triangleq W \varphi$. We will show that there exists $\varphi^* \in \R^{d_i}$ such that $\xi = W \varphi^*$ satisfies the conditions in Lemma~\ref{lm:aux-nn}. Then, using Lemma~\ref{lm:aux-nn} with the choice of $u=W\varphi^*$, $d=\diag(\nu) W b^*$, and $a=c_i$, we deduce that there exists $\alpha \in \R$ such that 
\begin{align}\label{eq:nn-constant-requirement-contrary}
     [\sdot(\alpha W \varphi^*)]^{\top} \cdot \diag(\nu) W b^* \neq c_i \ .
\end{align}
Hence, \eqref{eq:nn-constant-requirement} is false since it is violated for $\varphi=\alpha \varphi^*$ and the proof is completed. We show the existence of such $\varphi^*$ as follows. We first construct the set of $\varphi$ values for which conditions of Lemma~\ref{lm:aux-nn} on $w$ are \emph{not} satisfied. The set in question is the union of the following cases: (i) $\xi_i = W_i \varphi = 0$ for some $i \in [w]$, (ii) $|\xi_i| = |\xi_j|$ for some distinct $i,j \in [w]$, or equivalently, $(W_i \pm W_j) \varphi = 0$. Note that $W_{i} \neq \boldsymbol{0}$ and $W_{i} \pm W_{j} \neq \boldsymbol{0}$ by \eqref{eq:W_construction}. For a non-zero $y \in \R^{d_i}$, the set $\{\varphi \in \R^{d_i} \ : \ y^{\top} \varphi =  \}$ is a $(d_i-1)$-dimensional subspace of $\R^{d_i}$. Then, there are $w$ number of $(d_i-1)$-dimensional subspaces that fall under case (i), and $w(w-1)$ number of $(d_i-1)$-dimensional subspaces that fall under case (ii). Therefore, there are $w^2$ lower-dimensional subspaces for which conditions of Lemma~\ref{lm:aux-nn} do not hold. However, $\R^{d_i}$ cannot be covered by a finite number of lower-dimensional subspaces of itself. Therefore, there exists $\varphi^* \in \R^{d_i}$ such that $\xi = W \varphi^*$ satisfies the conditions of Lemma~\ref{lm:aux-nn}, and the proof is completed. \endproof

\subsection{Proof of Lemma~\ref{lm:aux-nn}}\label{sec:proof:lm:aux-nn}
Assume the contrary and suppose that there exists a non-zero $d$ and $a$ for a given $u$. Define the function $g_u(\alpha) \triangleq d^{\top} \sdot(\alpha u)$. Note that 
\begin{align}
    \sdot(x) = \frac{1}{1+e^{-x}+e^{x}}
\end{align}
is an even analytic function, and its Taylor series expansion at $0$ has the domain of convergence $\{x \in \R : |x| < \frac{\pi}{2}\}$. Thus, for all $\alpha \in (-\frac{\pi}{2 \max_{i} |u_{i}|}, \frac{\pi}{2 \max_{i} |u_{i}|})$,
\begin{align}
    a = g_u(\alpha) 
    &= \sum_{j=1}^{p} c_j \sdot(\alpha u_j) \\
    &= \sum_{i=0}^{\infty} \gamma_i \sum_{j=1}^{p} c_j (\alpha u_j)^{2i} \\
    &= \sum_{i=0}^{\infty} (\gamma_i \sum_{j=1}^{p} c_j u_j^{2i}) \alpha^{2i} \ .
\end{align}
Note that $g_u(\alpha)$ is a constant function of $\alpha \in (-\frac{\pi}{2 \max_{i} |u_{i}|}, \frac{\pi}{2 \max_{i} |u_{i}|})$. Thus, its Taylor coefficients, i.e., $(\gamma_i \sum_{j=1}^{p} c_j u_j^{2i})$, are zero for all $i \in \N^{+}$. However, the coefficients of even powers in Taylor expansion of $\sdot$, i.e., $\gamma_i$'s, are non-zero \citep{weisstein2002sigmoid}. Therefore, we have $\sum_{i=1}^{p}(c_j u_j^{2i}) = 0$ for all $i \in \N^+$. Next, construct the following system of linear equations,
\begin{align}
    \begin{bmatrix}
    u_1^2 & u_1^4 & \dots & u_1^{2w} \\
    \vdots & \vdots & \dots & \vdots \\
    u_p^2 & u_p^4 & \dots & u_p^{2w}
    \end{bmatrix} 
    \begin{bmatrix}
        d_1 \\ \vdots \\  d_p 
    \end{bmatrix}
    = \boldsymbol{0} \ .
\end{align}
This is equivalent to
\begin{align}
    \diag([u_1^2,\dots,u_p^2]^{\top}) 
    \underset{{\rm Vandermonde}}{\underbrace{
    \begin{bmatrix}
    1 & u_1^2 & \dots & u_1^{2(p-1)} \\
    \vdots & \vdots & \dots & \vdots \\
    1 & u_p^2 & \dots & u_p^{2(p-1)}
    \end{bmatrix}}}
    \begin{bmatrix}
        d_1 \\ \vdots \\  d_p 
    \end{bmatrix}
    =  \boldsymbol{0} \ . \label{eq:vandermonde}
\end{align}
Note that the Vandermonde matrix in \eqref{eq:vandermonde} has determinant $\prod_{1 \leq i \leq j < p}(u_i^2 - u_j^2)$, which is non-zero since $|u_i|\neq |u_j|$ for $i\neq j$. Multiplying an invertible matrix with a diagonal invertible matrix generates another invertible matrix, and $d$ must be zero vector which is a contradiction. Hence, there does not exist such $d$ for which $d^{\top}\sdot(\alpha u)$ is constant for every $\alpha \in \R$.

\section{Proofs of Identifiability Results}

\subsection{Proof of Proposition \ref{fact-permutation-diagonal}}\label{proof:fact-permutation-diagonal}
Denote the set of permutations of $[n]$ as
\begin{align}
    S_n \triangleq \{ \pi: \; P_\pi \in \Pi \} \ .
\end{align}
From Leibniz formula for matrix determinants, for a matrix $A \in \R^{n \times n}$ we have
\begin{align}
    \det(A) &= \sum_{\pi \in S_n} \sgn(\pi) \cdot \prod_{i = 1}^{n} A_{i, \pi_{i}}
\end{align}
where $\sgn(\pi)$ for a permutation $\pi$ of $[n]$ is $+1$ and $-1$ for even and odd permutations, respectively. $A$ is invertible if and only if $\det(A) \ne 0$, which implies
\begin{align}
    \exists \pi \in S_n \; : \;\; \sgn(\pi) \cdot \prod_{i = 1}^{n} A_{i, \pi_{i}} \ne 0 \ .
\end{align}
By the definition of $P_\pi$ in Definition \ref{def:valid-causal-order}, we have $[P_\pi A]_{i, i} = A_{i, \pi_{i}}$. Then,
\begin{align}
    \exists \pi \in S_n \; : \;\; \sgn(\pi) \cdot \prod_{i = 1}^{n} [P_\pi A]_{i, i} \ne 0
\end{align}
which holds if and only if $[P_\pi A]_{i, i} \neq 0$ for all $i \in [n]$.

\subsection{Proof of Lemma~\ref{lm:min-sparsity}} \label{appendix:lm:min-sparsity}
Permuting the rows or columns of a matrix does not change its $\ell_0$ norm. Also, based on the change notations in \eqref{eq:def-delta-X}-\eqref{eq:def-delta-Z-hat}, permuting rows of a matrix $A$ by $P \in \Pi$ before or after applying the change operator is essentially the same, i.e., $\Delta_{Z}(P \cdot A) = P \cdot \Delta_{Z}(A)$. Therefore, we have
\begin{align}
    \norm{\Delta_{\hat Z}(I_n)}_0
    &\overset{\eqref{eq:delta-z-hat-from-delta-z}}{=} \norm{\Delta_{Z}(H(U))} \\
    &= \norm{P_1(U) \cdot \Delta_{Z}(H(U))}_0 \\
    &= \norm{\Delta_{Z}(P_1(U) \cdot H(U))}_0 \\
    &\overset{\eqref{eq:def-H-bar}}{=} \norm{\Delta_{Z}(\bar H(U))}_0 \ . \label{eq:min-sparsity-step1}
\end{align}
For all $j, m \in [n]$ corresponding to which $[\Delta_{Z}(I_n)]_{j,m} = 1$, we get $[s_{Z}]_j \overset{p_Z}{\neq} [s^{m}_{Z}]_j$. By the definition of $\bar H(U)$ in \eqref{eq:def-H-bar}, $[\bar H(U)]_{j,j} \neq 0$. Then, using Assumption \ref{assumption:change}, we obtain
\begin{align}
    [\Delta_{Z}(\bar H(U))]_{j,m} = \mathds{1}\left([\bar H(U)]_{j} \cdot s_{Z} \overset{p_Z}{\neq}  [\bar H(U)]_{j} \cdot s^{m}_{Z} \right) = 1 \ . \label{eq:lm-ms-tmp1}
\end{align}
Since $\Delta_{Z}(I_n)$ and $\Delta_{Z}(\bar H(U)) \in \{0, 1\}^{n \times n}$, \eqref{eq:lm-ms-tmp1} implies that $\Delta_{Z}(I_n)  \preccurlyeq \Delta_{Z}(\bar H(U))$, which in conjunction with \eqref{eq:min-sparsity-step1}, we obtain $\norm{\Delta_{Z}(I_n)}_0 \leq \norm{\Delta_{Z}(\bar H(U))}_0 = \norm{\Delta_{\hat Z}(I_n)}_0$.

Next, we investigate the conditions under which \eqref{eq:lm:min-sparsity-p-outside} holds with equality. Since $[\Delta_{Z}(I_n)]_{j,m} = 1$ implies $[\Delta_{Z}(\bar H(U))]_{j,m} = 1$, the equality holds only if $[\Delta_{Z}(I_n)]_{j,m} = 0$ implies $[\Delta_{Z}(\bar H(U))]_{j,m} = 0$ as well. Therefore, $\norm{\Delta_{\hat Z}(I_n)}_0 = \norm{\Delta_{Z}(I_n)}_0$ is satisfied only if 
\begin{align}
    \Delta_{Z}(I_n) &= \Delta_{Z}(\bar H(U)) \\
    &= \Delta_{Z}(P_1 (U) \cdot H(U)) \label{eq:delta-z-H-bar} \\
    &\overset{\eqref{eq:delta-z-hat-from-delta-z}}{=} \Delta_{\hat Z}(P_1(U)) \\
    & = P_1(U) \cdot \Delta_{\hat Z}(I_n) \ ,
\end{align}
which concludes the proof. 

\subsection{Proof of Lemma \ref{lm:rho}}\label{proof:rho}
Since $U \in \mcU_{1}$, Lemma \ref{lm:min-sparsity} shows that $\Delta_{Z}(I_n) = P_1(U) \cdot \Delta_{\hat Z}(I_n)$, which is equal to $\Delta_{Z}(\bar H(U))$ as specified in \eqref{eq:delta-z-H-bar}. By definition, $\bar H(U)$ has non-zero diagonal entries which ensures $I_n \preccurlyeq \mathds{1}(\bar H(U))$. Next, we show that $[\mathds{1}(\bar H(U))]_{i,j} \leq \Sigma_{i,j}$ for all $i, j \in [n]$. If $\overline{\Ch}(j) \subseteq \Ch(i)$, then $\Sigma_{i,j} = 1$ and $[\mathds{1}(\bar H(U))]_{i,j}\leq \Sigma_{i,j}$. Thus, we just need to check the cases in which $\overline{\Ch}(j) \setminus \Ch(j) \neq \emptyset$.

Let $k \in \overline{\Ch}(j) \setminus \Ch(i)$ and consider environment $\mcE^{m_k}$ in which node $k$ is intervened. Since $j \in \overline{\Pa}(k)$ and $i \notin \overline{\Pa}(k)$, by Lemma~\ref{lm:parent_change} we have $[\Delta_{Z}(I_n)]_{j,m_k} = 1$ and $[\Delta_{Z}(I_n)]_{i,m_k} = 0$. Let $[\bar H(U)]_{i,j} \neq 0$. Since $[\Delta_{Z}(I_n)]_{j,m_k1}=1$ means that $[s_{Z}]_j \overset{p_Z}{\neq} [s^{m_k}]_j$, by Assumption \ref{assumption:change} we have
\begin{align}
    \left( [\bar H(U)]_i \cdot s_{Z} \right) \overset{p_Z}{\neq} \left( [\bar H(U)]_i \cdot s^{m_k} \right) \ ,
\end{align}
which means $[\Delta_{Z}(\bar H(U))]_{i,m_k} = 1$. Since we have $\Delta_{Z}(I_n) = \Delta_{Z}(\bar H(U))$, we also should have $[\Delta_{Z}(I_n)]_{i,m_k}=1$, which contradicts with $[\Delta_{Z}(I_n)]_{i,m_k}=0$. Therefore, if $\overline{\Ch}(j) \setminus \Ch(i) \neq \emptyset$, which means $\Sigma_{i,j} = 0$, we have $[\bar H(U)]_{i,j}=0$, and we conclude that $\mathds{1}(\bar H(U))\preccurlyeq I_n + \Sigma$. This result means that given a $U \in \mcU_1$, $\bar H(U)$ can be decomposed as
\begin{align}
    \bar H(U) = C \cdot (I_n + \Lambda) \ , \label{eq:lm-rho-e2}
\end{align}
where $C$ is a diagonal matrix and $\Lambda$ is a strictly upper triangular matrix that satisfies $\mathds{1}(\Lambda)\preccurlyeq \mathds{1}(\Sigma)$. Subsequently, from \eqref{eq:lm-rho-e2} we have
\begin{align}
    [\bar H(U)]^{-1} &= (I_n + \Lambda)^{-1} \cdot C^{-1} \ ,
\end{align}
and
\begin{align}
    \mathds{1}([\bar{H(U)}]^{-1}) &= \mathds{1}((I_n+\Lambda)^{-1}) \ ,
\end{align}
since $C$ is a diagonal matrix. Note that $\Lambda$ is strictly upper triangular $n \times n$ matrix, which implies $\Lambda^{n}$ is a zero matrix. Therefore, the inverse of $I_n + \Lambda$ can be expanded as
\begin{align}
    (I_n+\Lambda)^{-1} = I_n - \Lambda + \Lambda^{2} - \dots + (-1)^{n-1}\Lambda^{n-1} \ . \label{eq:inverse-expansion}
\end{align}
If $[I_n+\Lambda]^{-1}_{i,j} \neq 0$ for some $i, j \in [n]$ with $i < j$, then by \eqref{eq:inverse-expansion} we have $[\Lambda^{k}]_{i,j} \neq 0$ for some $k \in [n-1]$. Expanding matrix $\Lambda^{k}$ yields that $[\Lambda^{k}]_{i,j}$ is equal to sum of the products with $k$ terms, i.e.,
\begin{align}
   [\Lambda^{k}]_{i,j} = \sum_{i < a_1 <\dots < a_{k-1} < j} [\Lambda]_{i,a_1} [\Lambda]_{a_1,a_2} \dots [\Lambda]_{a_{k-1},j} \ .
\end{align}
Therefore, if $[\Lambda^{k}]_{i,j} \neq 0$, there exists a sequence of entries $([\Lambda]_{i,a_1},[\Lambda]_{a_1,a_2},\dots,[\Lambda]_{a_{k-1},j})$ in which all terms are non-zero. Since $\mathds{1}(\Lambda)\preccurlyeq \mathds{1}(\Sigma)$, this also implies that $\Sigma_{i,a_1}=\Sigma_{a_1,a_2}=\dots=\Sigma_{a_{k-1},j}=1$. Furthermore, by the definition of $\Sigma$, $\Sigma_{u,v}=1$ means that $\overline{\Ch}(v) \subseteq \Ch(u)$, and we have
\begin{align}
    \overline{\Ch}(j) \subseteq \Ch(a_{k-1}) \subset \dots \subset \Ch(a_1) \subset \Ch(i) \ ,
\end{align}
which also implies $\Sigma_{i,j}=1$ since $\overline{\Ch}(j) \subseteq \Ch(i)$. Therefore, we deduce that if $[(I_n+\Lambda)^{-1}]_{i,j}\neq 0$ for any $i, j \in [n]$, then we have $\Sigma_{i,j}=1$. Hence, $\mathds{1}([\bar H(U)]^{-1}) \preccurlyeq I_n + \Sigma$. Taking the transpose concludes the proof. 

\subsection{Proof of Theorem \ref{th:soft-final-results}} \label{proof:soft-final-results}
\paragraph{1 -- Equality of DAGs.} By the definition of $\mcU_{\rm S}$ in \eqref{eq:algo-mcus-def}, for every $U \in \mcU_{\rm S}$ there exists a permutation matrix $P_2(U)$ such that
\begin{align}\label{eq:K}
    K(U)
    &\overset{\eqref{eq:def-K}}{=} \Delta_{X}(U^{\top}) \cdot P_2(U) \\
    &\overset{\eqref{eq:delta-z-hat-from-delta-x}}{=} \Delta_{\hat Z}(I_n) \cdot P_2(U) \\
    &\overset{\eqref{eq:lm:min-sparsity-p-outside}}{=} [P_1(U)]^{\top} \cdot \Delta_{Z}(I_n) \cdot P_2(U) \\
    &\overset{\eqref{eq:def-D-prime}}{=} [P_1(U)]^{\top} \cdot D \cdot P_{\mcI} \cdot P_2(U) \label{eq:soft-results-K-initial}
\end{align}
is an upper triangular matrix. For a given $U \in \mcU_{S}$, let $\pi$ and $\bar{\pi}$ be the permutations of $[n]$ that describe the row and column operations of $P_1(U)$ and $P_{\mcI} \cdot P_2(U)$, respectively, such that
\begin{align}
    [P_1(U)]_{i,\pi_i} &=1 \ , \quad \text{and} \quad
    [P_{\mcI} \cdot P_2(U)]_{i,\bar{\pi}_1} = 1 \ , \quad  \forall i \in [n] \ . \label{eq:soft-res-p1-pi-rel}
\end{align}
Then, the entries of $D$ and $K(U)$ are related through
\begin{align}
    D_{i,j}
    &= [[P_1(U)]^{\top} \cdot D]_{\pi_i,j} \\
    &= [[P_1(U)]^{\top} \cdot D \cdot P_{\mcI} \cdot  P_2(U)]_{\pi_i,\bar{\pi}_j} \\
    &= [K(U)]_{\pi_i,\bar{\pi}_j} \ .
\end{align}
Using Lemma~\ref{lm:parent_change}, $D_{i,i} = [\Delta_Z(I_n)]_{i,m_i}  = 1$ for all $i \in [n]$, which indicates $[K(U)]_{\pi_i,\bar{\pi}_i} = 1$. Since $D$ is upper triangular, we must have $\pi_i \leq \bar \pi_i$ in order for $[K(U)]_{\pi, \bar \pi} = 1$ to hold for all $i \in [n]$. Since $\pi$ and $\bar \pi$ are permutations of $[n]$, this implies $\pi = \bar \pi$. Therefore, we have
\begin{align}
    P_1(U) &= P_{\mcI} \cdot P_2 (U) \ ,
\end{align}
and
\begin{align} \label{eq:soft-K}
    K(U) &\overset{\eqref{eq:soft-results-K-initial}}{=} [P_1(U)]^{\top} \cdot D \cdot P_1(U) \ .
\end{align}
Next, if $i \in \overline{\Pa}(j)$, Lemma~\ref{lm:parent_change} shows that $D_{i,j}=1$, which implies $[K(U)]_{\pi_i,\pi_j}=1$. Since $K(U)$ is upper triangular, we have $\pi_i \leq \pi_j$, which implies that $\pi$ is a valid causal order. Now, consider the randomly chosen transformation $\hat T$ from $\mcU_{\rm S}$, and recall the construction of $\mcG_{\hat Z}$ in \eqref{eq:DAG-construction}, i.e.,
\begin{align}
    \mbox{for all nodes $i\in\mcG_{\hat Z}$: } \quad \Pa(i) \triangleq \{j : [K(\hat T)]_{j,i} = 1\}\ .
\end{align}
Therefore, $\mcG_{\hat Z}$ is equal to $\mcG_{Z}$ under permutation $\pi$, which maps the nodes of $\mcG_{Z}$ to those of $\mcG_{\hat Z}$]. {We will show the last part, that is $\hat Z$ is Markov with respect to $\mcG_{\hat Z}$ after showing the mixing consistency.}

\paragraph{2 -- Mixing consistency of the estimates.} We start with specifying the auxiliary estimate $\hat Z(U,X)$ in terms of the true latent variables. Note that the estimate of the latent variables under any $U\in\mcU_1$ in \eqref{eq:z-hat-from-z} can be written as
\begin{align}
    \hat Z(U,X) 
    &\overset{\eqref{eq:def-H-bar}}{=} [P_1(U)]^{\top} \cdot [\bar H(U)]^{-\top} \cdot Z \label{eq:rho-prenext-12} \ .
\end{align}
We defined $\pi$ as the permutation related to $[P_1(U)]^{\top}$ in \eqref{eq:soft-res-p1-pi-rel}, based on which $P_\pi = [P_1(U)]^{\top}$. Also, from Lemma~\ref{lm:rho} we know that $[\bar H(U)]^{-\top}$ satisfies $I_n \preccurlyeq \mathds{1}([\bar H(U)]^{-\top}) \preccurlyeq I_n + \Sigma^{\top}$, where $\Sigma$ is a strictly upper triangular matrix defined in \eqref{eq:rho}. As such, we can expand $[\bar H(U)]^{-\top}$ as
\begin{align}\label{eq:C-and-B}
    [\bar H(U)]^{-\top} = C + B
\end{align}
where $C$ is a full-rank $n \times n$ diagonal matrix and $B$ satisfies $\mathds{1}(B) \preccurlyeq \Sigma^{\top}$. Using the definition of $\Sigma$ in \eqref{eq:rho}, this implies that
\begin{align}
    j \notin \sur(i) \implies B_{i,j} = 0 \ .
\end{align}
Using these relations, we can write \eqref{eq:rho-prenext-12} as
\begin{align}
    \hat Z(U,X) = P_\pi \cdot (C + B) \cdot Z \ . \label{eq:rho-tmp3}
\end{align}
By setting $U = \hat T$ in \eqref{eq:rho-tmp3}, we obtain
\begin{align}\label{eq:hat-z-from-z}
    \hat Z &= P_\pi \cdot (C + B) \cdot Z \ .
\end{align}
This result states that for any realization $Z^t$ of $Z$ and its associated estimate $\hat Z^t$, we have
\begin{align}
    \hat Z^t &= P_\pi \cdot (C + B) \cdot Z^t \ ,
\end{align}
and the proof of mixing consistency is concluded. We investigate the properties of this result further to prove that $\hat Z$ is Markov with respect to $\mcG_{\hat Z}$. First, we pre-multiply both sides of \eqref{eq:rho-prenext-12} by $P_1(U)$ to obtain
\begin{align}
    P_1(U) \cdot \hat Z(U,X) &= [\bar H(U)]^{-\top} \cdot Z \ . \label{eq:rho-next-1}
\end{align}
For any $U \in \mcU_1$, Lemma~\ref{lm:rho} shows that $[\bar H(U)]^{-\top}_{i, j} = 0$ if $j \not\in \sur(i)$. Therefore, for row $i \in [n]$ in \eqref{eq:rho-next-1}, we get 
\begin{align}
     [P_1(U) \cdot \hat Z(U,X)]_i &=
     [\bar H(U)]^{-\top}_{i,i} Z_i + \sum_{j \in \sur(i)} [\bar H(U)]^{-\top}_{i,j} Z_j \ . \label{eq:rho-next-2}
\end{align}
Let us restate \eqref{eq:rho-next-2} for $U = \hat T$ in terms of matrices $P_\pi$, $C$, and $B$ as
\begin{align}
     [[P_\pi]^{\top} \cdot \hat Z]_i &= [C + B]_{i} \cdot Z \ ,
\end{align}
or equivalently,
\begin{align}
     [\hat Z]_{\pi_i} &= C_{i,i} Z_i + \sum_{j \in \sur(i)} B_{i,j} Z_j \ . \label{eq:rho-next-3}
\end{align}
Note that if node $i$ is non-surrounded in $\mcG_{Z}$, that is $\sur(i) = \emptyset$, the sum of terms $B_{i,j} Z_j$ in \eqref{eq:rho-next-3} becomes zero, and we obtain
\begin{align}
    [\hat Z]_{\pi_i} = C_{i,i} Z_i \ , \quad \forall i \notin \mcS \ .
\end{align}
Denote the deterministic function that corresponds to the causal mechanism of $i$ by $f_i$ as
\begin{align}
    Z_i &= f_i(Z_{\Pa(i)},R_i) \ ,
\end{align}
in which $R_i$ is an exogenous random variable. Since $\sur(i) \subseteq \Pa(i)$, for each $i \in [n]$, we can write
\begin{align}
    [\hat Z]_{\pi_i} &= C_{i,i} Z_i + \sum_{j \in \sur(i)} B_{i,j} Z_j \\
    &= C_{i,i} f_i(Z_{\Pa(i)},R_i) + \sum_{j \in \sur(i)} B_{i,j} Z_j \\
    &= h_i(Z_{\Pa(i)},R_i) \ ,
\end{align}
for deterministic functions $\{h_i : i \in [n]\}$. But, we need to write $[\hat Z]_{\pi_i}$ in terms of $\hat Z_{\Pa(\pi_i)}$ instead of $Z_{\Pa(i)}$. It suffices to show that for any $k \in \Pa(i)$, $Z_{k}$ is a function of $\hat Z_{\Pa(\pi_i)}$ for this result. To show that, first use \eqref{eq:hat-z-from-z} to obtain
\begin{align}
    Z &= (C+B)^{-1} \cdot [P_{\pi}]^{\top} \cdot \hat Z  \label{eq:z-from-hat-z-pre} \\
    \overset{\eqref{eq:C-and-B}}&{=} [\bar H(\hat T)]^{\top} \cdot [P_{\pi}]^{\top} \cdot \hat Z \ ,
\end{align}
By Lemma \ref{lm:rho}, $I_n \preccurlyeq \mathds{1}([\bar H(\hat T)]^{\top})  \preccurlyeq I_n  + \Sigma$, and we have
\begin{align}
    Z &= (C^{-1} + E) \cdot  [P_{\pi}]^{\top} \cdot \hat Z \label{eq:z-from-hat-z}
\end{align}
where $E$ satisfies $\mathds{1}(E) \preccurlyeq \Sigma$. Therefore, as counterpart of \eqref{eq:rho-next-3}, we have
\begin{align}
    [Z]_k &= [C^{-1}]_{k,k} \hat Z_{\pi_k} + \sum_{j \in \sur(k)} E_{k,j} \hat Z_{\pi_j} \ . \label{eq:rho-next-4}
\end{align}
Since $k \in \Pa(i)$, $\pi_k \in \Pa(\pi_i)$ and $\hat Z_{\pi_k}$ is in $\hat Z_{\Pa(\pi_i)}$. Note that if $j \in \sur(k)$, $j$ is also in $\Pa(i)$. Therefore, every term in \eqref{eq:rho-next-4} belongs to $\hat Z_{\Pa(\pi_i)}$ and $Z_k$ is a function of $\hat Z_{\Pa(\pi_i)}$. Then, we have
\begin{align}
    \hat Z_{\pi_i} = g_i(\hat Z_{\Pa(\pi_i)},R_i)
\end{align}
for deterministic functions $\{g_i : i \in [n]\}$ and $\hat Z$ is Markov with respect to $\mcG_{\hat Z}$.

\begin{corollary}\label{corollary:soft-K}
For the chosen estimate $\hat T \in \mcU_{\rm S}$, we have
    \begin{align}
         K(\hat T) = [P_1(\hat T)]^{\top} \cdot D \cdot P_1(\hat T) \ . 
    \end{align}
\end{corollary}
\begin{proof}
    This is a direct result of Theorem~\ref{th:soft-final-results}, as shown in \eqref{eq:soft-K} in the proof.
\end{proof}

\subsection{Proof of Proposition~\ref{prop:same-dag}}\label{proof:same-dag}
We show this by noting that for any $U\in\mcU_{2}$ we have
\begin{align}
    \Delta_{\hat Z(U,X)}(I_n)
    \overset{\eqref{eq:delta-z-hat-from-delta-x}}{=} \Delta_{X}(U^{\top})
    \overset{\eqref{eq:algo-mcu2-def}}{=} \Delta_{X}(\hat T^{\top})
    \overset{(\ref{eq:delta-z-hat-from-delta-x}, \ref{eq:encoder2})}{=} \Delta_{\hat Z}(I_n) \ .
\end{align}
Subsequently,
\begin{align}
    P_1(U) \cdot \Delta_{\hat Z(U,X)}(I_n) \overset{\eqref{eq:lm:min-sparsity-p-outside}}{=} \Delta_{Z}(I_n)  \overset{(\ref{eq:lm:min-sparsity-p-outside}, \ref{eq:encoder2})}{=} P_1(\hat T) \cdot \Delta_{\hat Z}(I_n) \ .
\end{align}
Note that $\Delta_{Z}(I_n) = D \cdot P_{\mcI}$ is invertible since $D$ is an upper triangular matrix with non-zero diagonal entries. Then, $\Delta_{\hat Z(U,X)}(I_n) = \Delta_{\hat Z)}(I_n)$ is also invertible, which implies that $P_1(U) = P_1(\hat T)$. Then, by Corollary~\ref{corollary:soft-K} we have $K(U) = K(\hat T)$, and the DAGs constructed from $K(U)$ and $K(\hat T)$ are the same.

\subsection{Proof of Theorem \ref{th:hard-final-results}} \label{proof:hard-final-results}
For proving the scaling consistency, we only need to consider surrounded nodes since by Theorem~\ref{th:soft-final-results}, we can already recover the non-surrounded nodes up to scaling under the valid causal order $\pi$. Consider a surrounded node $\pi_i \in \mcS_{\pi}$ and node $\pi_j \in \sur(\pi_i)$ in graph $\mcG_{\hat Z}$. These nodes correspond to nodes $i$ and $j$ in the true latent DAG $\mcG_{Z}$, respectively, with $j \in \sur(i)$. From \eqref{eq:rho-next-3} in the proof of Theorem~\ref{th:soft-final-results} in Appendix~\ref{proof:soft-final-results}, we have
\begin{align}
    [\hat Z]_{\pi_i} &= C_{i,i} Z_i + \sum_{k \in \sur(i)} B_{i,k} Z_k \ , \label{eq:hard-results-tmp1} \\
    [\hat Z]_{\pi_j} &= C_{j,j} Z_j + \sum_{k \in \sur(j)} B_{j,k} Z_k \ . \label{eq:hard-results-tmp2}
\end{align}
In \eqref{eq:hard-results-tmp1}, if the entry $B_{i,j}$ is non-zero, then $[\hat Z]_{\pi_i}$ includes $Z_j$. From \eqref{eq:hard-results-tmp2}, $[\hat Z]_{\pi_j}$ always includes $Z_j$, thus $[\hat Z]_{\pi_i}$ and $[\hat Z]_{\pi_j}$ cannot be independent in any environment. Therefore, by enforcing $[\hat Z^{m}]_{\pi_i} \ci [\hat Z^{m}]_{\pi_j}$ to hold in at least one environment $m \in [n]$, we force $B_{i,j}$ to be zero. We know that such an environment always exists since by choosing $m = m_i$ and using Proposition~\ref{fact:hard-indep}, we get
\begin{align}
    Z_i^{m_i} \ci Z_j^{m_i} \ , \quad \forall i \in \mcS,\ j \in \sur(i) \ .
\end{align}
This implies that $[\hat Z^{m}]_{\pi_i} \ci [\hat Z^{m}]_{\pi_j}$ if $B$ is a zero matrix. Then, exhausting all possible $i,j \in [n]$ values, we force all entries of $B$ to be zero. The desired follows as, for any realization $Z^t$ of $Z$ and its associated estimate $\hat Z^t$, we have
\begin{align}
    \hat Z^t &\overset{\eqref{eq:soft-final-results}}{=} P_\pi \cdot (C + B) \cdot Z^t = P_\pi \cdot C \cdot Z^t \ .
\end{align}

\bibliography{references}
\bibliographystyle{plainnat}

\end{document}